\documentclass[twoside,11pt]{article}

\usepackage{color}
\usepackage{template}
\usepackage{graphics}
\usepackage{epsfig}
\usepackage[ruled]{algorithm2e}
\usepackage{booktabs}
\usepackage{url}
\usepackage{psfrag}
\usepackage{float}
\usepackage{subfigure}
\usepackage[stable]{footmisc}
\usepackage{amsmath,amsfonts,mathrsfs,bm}
\usepackage{algorithmic}
\usepackage[breaklinks=true]{hyperref}
\usepackage{relsize}
\usepackage{lastpage}
\usepackage[format=hang,justification=justified,singlelinecheck=false,font=footnotesize]{caption}

\newtheorem{assumption}{Assumption}

\numberwithin{equation}{section}

\usepackage{tikz,pgfplots}
  \pgfplotsset{compat=newest}
  \pgfplotsset{plot coordinates/math parser=false,trim axis left}
     \newlength\figureheight
    \newlength\figurewidth
    \usetikzlibrary{shapes,arrows}

\begin{document}

\title{A Statistical Learning Approach to Modal Regression}

\author{\name Yunlong Feng \email ylfeng@albany.edu\\
        \addr Department of Mathematics and Statistics\\
        State University of New York\\
        The University at Albany\\
        Albany, New York 12222, USA
\medskip\\
\name Jun Fan \email junfan@hkbu.edu.hk\\
\addr Department of Mathematics\\
Hong Kong Baptist University\\
 Kowloon, Hong Kong, China
\medskip\\
\name Johan A.K. Suykens \email johan.suykens@esat.kuleuven.be\\
       \addr Department of Electrical Engineering\\ESAT-STADIUS, KU Leuven\\
       Kasteelpark Arenberg 10, Leuven\\
       B-3001, Belgium}

\maketitle

\begin{abstract}
\hspace{-0.18cm}This paper studies the nonparametric modal regression problem systematically from a statistical learning viewpoint.  Originally motivated by pursuing a theoretical understanding of the maximum correntropy criterion based regression (MCCR), our study reveals that MCCR with a tending-to-zero scale parameter is essentially modal regression. We show that the nonparametric modal regression problem can be approached via the classical empirical risk minimization. Some efforts are then made to develop a framework for analyzing and implementing modal regression. For instance, the modal regression function is described, the modal regression risk is defined explicitly and its \textit{Bayes} rule is characterized; for the sake of computational tractability, the surrogate modal regression risk, which is termed as the generalization risk in our study, is introduced. On the theoretical side, the excess modal regression risk, the excess generalization risk, the function estimation error, and the relations among the above three quantities are studied rigorously. It turns out that under mild conditions, function estimation consistency and convergence may be pursued in modal regression as in vanilla regression protocols such as mean regression, median regression, and quantile regression. On the practical side, the implementation issues of modal regression including the computational algorithm and the selection of the tuning parameters are discussed. Numerical validations on modal regression are also conducted to verify our findings.   

 \end{abstract}
 
\section{Introduction}
In this paper, we are interested in the nonparametric regression problem which aims at inferring the functional relation between input and output. Regression problems are concerned with the conditional distribution, which in practice can never be known in advance. Instead, normally, what one can access is only a set of observations drawn from the joint probability distribution. To state this problem mathematically, let us denote $X$ as the explanatory variable that takes values in a compact metric space $\mathcal{X}\subset\mathbb{R}^d$ and $Y$ that takes values in $\mathcal{Y}=\mathbb{R}$ as the response variable. Typically, we consider the following  data-generating model $$Y=f^\star(X)+\epsilon,$$ where $\epsilon$ is the noise variable. In nonparametric regression problems, the purpose is to infer the unknown function $f^\star$ nonparametrically while certain assumptions on the noise  variable $\epsilon$ may be imposed. As a compromise, regression estimators usually settle for learning a characterization of the conditional distribution by sifting information through observations generated above. Characterizations of the conditional distribution are versatile, where the several usual ones include the conditional mean, the conditional median, the conditional quantile, and the conditional mode. The versatility of the characterizations of the conditional distribution raises the question that which characterization we should pursue in regression problems. To answer this question, tremendous attention has been drawn in the statistics and machine learning communities. As a matter of fact, a significant part of parametric and nonparametric regression theory has been fostered  to illuminate this question.

It is generally considered that each of the above-mentioned regression protocols has its own merits in its own regimes. For instance, it has been well understood that regression towards the conditional mean can be most effective if the noise is Gaussian or sub-Gaussian. Regression towards the conditional median or conditional quantile can be more robust in the absence of light-tailed noise or symmetric conditional distributions. In practice, the choice of the most appropriate regression protocol is usually decided by the type of data encountered. In the statistics and machine learning literature, these regression protocols have been studied extensively and understood well. In this study, we focus on a regression problem that has not been well studied in the statistical learning literature, namely, modal regression.

\subsection{Modal Regression} 
Modal regression approaches the unknown truth $f^\star$ by regressing towards the \textit{conditional mode function}. For a set of observations, the mode is the value that appears most frequently. While for a continuous random variable, the mode is the value at which its density function attains its peak value. The conditional mode function is denoted pointwisely as the mode of the conditional density of the dependent variable conditioned on the independent variable.

Previously proposed in \cite{sager1982maximum,collomb1987note} and studied in, e.g., \cite{lee1989mode,lee1993quadratic}, it is shown that one of the most appealing features of modal regression lies in its robustness to outliers, heavy-tailed noise, and skewed noise. Moreover, regression towards the conditional mode in some cases can be a better option when predicting the trends of observations. This is also the case in some real-world applications, as illustrated in \cite{matzner1998nonparametric,einbeck2006modelling,yu2014fast}.  However, it seems to us that so far not enough attention has been given to the theory and applications of modal regression, especially in the statistical learning literature. As yet another regression protocol, the above-mentioned merits of modal regression suggest that it deserves far more attention than it has received, especially in the big data era today. This motivates our study on modal regression in this paper.

\subsection{Historical Notes on Modal Regression}
Modal regression is concerned with the mode. Studies on the mode estimation date back to the 1960s since the seminal work of \cite{parzen1962estimation}. It opens the door for kernel density estimation by proposing the \textit{Parzen window} method, with the help of which the estimation of the mode can typically proceed. Many subsequent studies concerning theoretical as well as practical estimation of the mode have been emerging since then, among them are \cite{chernoff1964estimation,robertson1974iterative,fukunaga1975estimation,eddy1980optimum,comaniciu2002mean}, and \cite{dasgupta2014optimal}.
 
In a regression setup, the concern of the conditional mode estimate gives birth to modal regression. As far as we are aware, the idea of regression towards the conditional mode was first proposed in \cite{sager1982maximum} in an isotonic regression setup. It was then specifically investigated in \cite{collomb1987note} when dealing with dependent observations. As a theoretical study, the main conclusion drawn there was the uniform convergence of the nonparametric mode estimator to the conditional mode function. Lately, in \cite{lee1989mode,lee1993quadratic}, some pioneering studies of modal regression were conducted. The tractability problem of mode regression was first discussed in their studies from, say, a supervised learning and risk minimization viewpoint. By considering some specific modal regression kernels, and assuming the existence of a global conditional mode function under a linear model assumption, they established the asymptotic normality of the resulting estimator. More and more attention to the theory and applications of modal regression has been attracted since the work in \cite{yao2012local,yao2014new} and \cite{kemp2012regression}. In \cite{yao2014new}, a global mode was assumed to exist and take a linear form. Under proper assumptions on the conditional density of the noise variable, the implementation issues and the asymptotic normality of the estimator, as well as its robustness were explored. Recently, \cite{chen2016nonparametric} presented an interesting study towards modal regression in which the conditional mode was sought by estimating the maximum of a joint density. By assuming a factorizable modal manifold collection, results on asymptotic error bounds as well as techniques for constructing confidence sets and prediction sets were provided. 

To further disentangle the literature on modal regression, we can roughly categorize existing studies by tracing the thread of global or local approaches that they follow. For local approaches, the conditional mode is sought via maximizing a conditional density or a joint density which is typically estimated non-parametrically, e.g., by using kernel density estimators. Studies in \cite{collomb1987note,samanta1990non,quintela1997nonparametric,ould1997note,herrmann2004rates,ferraty2005functional,gannoun2010semiparametric,yao2012local,chen2016nonparametric,sasaki2016modal,zhou2016nonparametric,yao2016nonparametric,zhou2017bandwidth} fall into this category. For global approaches, the conditional mode is usually sought by maximizing the kernel density estimator for the variable induced by the residual and assuming that the global mode is unique and belongs to a certain hypothesis space. To name a few, studies in \cite{lee1989mode,lee1993quadratic,lee1998semiparametric,yao2014new,kemp2012regression,baldauf2012use,yu2012bayesian,lv2014robust,salah2016nonlinear} follow this line. It should be noticed that most studies based upon global approaches assume the existence (and also the uniqueness) of a global conditional mode function that is of a parametric form. While for the studies based upon local approaches, usually only the uniqueness assumption of the conditional mode function is imposed. Loosely speaking, modal regression estimators of the former case are nonparametric, while (semi-) parametric in the latter case.     

Most of the above-mentioned studies are theoretical in nature. It should be noted that some application-oriented studies on modal regression have also been conducted. Among them, \cite{matzner1998nonparametric} carried out an empirical comparison among three regression schemes, namely, the conditional mean regression, the conditional median regression, and the conditional mode regression, in nonparametric forecasting problems. They empirically observed that for certain datasets, e.g., the Old Faithful eruption prediction dataset, the mode can be a better option in forecasting than the mean and the median; \cite{yu2014fast} discussed the mode-based regression problem in the big data context. Based on empirical evaluations on the Health Survey for England dataset, they argued that the mode could be an effective alternative for pattern-finding; \cite{einbeck2006modelling} dealt with the speed-flow data in traffic engineering by applying a multi-modal regression model.

\subsection{Objectives of This Study and Our Contributions}
As mentioned above, in the statistics literature, there exist some interesting studies towards modal regression from both theoretical and practical viewpoints. However, we notice that several problems related to the theoretical understanding as well as the practical implementations of modal regression remain unclear. For example:
\begin{itemize}
\item Modal regression regresses towards the conditional mode function, a direct estimation of which involves the estimation of a  conditional or joint density. In fact, many of the existing studies on modal regression follow this approach. Notice that the explanatory variable may be high-dimensional vector-valued, which may make the estimation of the conditional or the joint density infeasible. This poses an important question: how to carry out modal regression without involving the estimation of a density function in a (possibly) high-dimensional space? According to the existing studies on modal regression, assuming the existence of a global conditional mode function and imposing some prior structure assumptions on it seem to be promising in avoiding estimating such a density. However, most existing studies of this type assume that the conditional mode function possesses a certain linear or parametric form. This could be restrictive in certain circumstances.
\item With a modal regression estimator at hand, how can we evaluate its statistical performance? That is, how can we measure the approximation ability of the modal regression estimator to the conditional mode function? This concern is of great importance in nonparametric statistics as well as in machine learning as it is closely related to the prediction ability of the estimator on future observations. On the other hand, concerning the implementation issues of modal regression, how can we perform model selection in modal regression? 
\end{itemize}

To address the above two problems raised in modal regression, in this study, we propose to perform modal regression through the classical empirical risk minimization (ERM) scheme. Within the statistical learning framework, we then develop a learning theory framework for assessing the performance of the resulting modal regression estimator. Our contributions made in this study can be summarized as follows:
\begin{itemize}
\item The first main contribution of our study is that we present the first systematic statistical learning treatment on modal regression. This purpose is achieved by developing a statistical learning setup for modal regression, adapting it into the classical ERM framework, and conducting a learning theory analysis for modal regression estimators. The statistical learning approach to modal regression in this paper distinguishes our work from previous studies.
\item The second main contribution of this study lies in that we develop a statistical learning framework for modal regression. To this end, the modal regression risk is devised, the \textit{Bayes} rule of the modal regression risk is characterized, computationally tractable surrogates of the modal regression risk are introduced, and ERM schemes for modal regression are formulated.  
\item Following the ERM scheme, by assuming the existence of a global conditional mode function, the modal regression estimator in our study is pursued by maximizing a one-dimensional density estimator. This is more computationally tractable compared with the approaches adopted in most of the existing studies, in which the estimation of a possibly high-dimensional density is involved, as detailed in Section \ref{subsec::comparison}. This gives the third main contribution of this study.        
\item Another contribution made in this paper is that we present a learning theory analysis on the modal regression estimator resulted from the ERM scheme. The theoretical results in our analysis are concerned with the modal regression risk consistency, the generalization risk consistency, the function estimation ability of the modal regression estimator, and their relations, see Section \ref{sec::learning_theory} for details.   
\item It should be highlighted that, as we shall also explain below, the study in this paper is originally motivated by pursuing some further understanding of the maximum correntropy based regression (MCCR), which was recently investigated in \cite{fenglearning}. In particular, this study is started with the realization that MCCR with a tending-to-zero scale parameter is modal regression, see Section \ref{sec::correntropy} for details. It turns out that the study conducted in this paper brings us some new perspectives and a deeper understanding of MCCR.
\end{itemize}

\begin{table}[t] 
\setlength{\tabcolsep}{15pt}
\setlength{\abovecaptionskip}{5pt}
\setlength{\belowcaptionskip}{5pt}
\centering
\captionsetup{justification=centering}
\vspace{.5em}
\begin{tabular}{@{} ll@{}}
  \toprule
    notation & meaning  \\ \midrule
     $\mathcal{X}$, $\mathcal{Y}$ & the independent variable space and the dependent variable space, respectively\\
    $X, Y$ & random variables taking values in $\mathcal{X}$ and $\mathcal{Y}$, respectively\\
    $x,y$ & realizations of $X$ and $Y$, respectively\\ 
    $\mathcal{M}$ & the function set comprised of all measurable function from $\mathcal{X}$ to $\mathbb{R}$\\
        $\epsilon$ & the noise variable specified by the residual $Y-f^\star(X)$\\ 
        $\mathbf{z}$ & a set of $n$-size realizations of $(X,Y)$ with $\mathbf{z}:=\{(x_i,y_i)\}_{i=1}^n$ \\
    $E_f$ & the random variable induced by the residual $Y-f(X)$ \\
    $\mathcal{H}$ & a hypothesis space that is assumed to be a compact subset of $C(\mathcal{X})$\\
    $K_\sigma$ & a smoothing kernel with the bandwidth $\sigma$ \\
    $\rho$ &  the joint probability distribution of $X\times Y$ \\
    $\rho_{\mathsmaller{\mathcal{X}}}$ & the marginal distribution of $X$ \\
    $L_{\rho_{\mathsmaller{\mathcal{X}}}}^2$ & the function space of square-integrable functions with respect to   $\rho_{\mathsmaller{\mathcal{X}}}$\\
    $p_\mathsmaller{E_f}$ or $p_\mathsmaller{f}$   & the density function of the random variable $E_f$ \\
      $p_{\mathsmaller{Y|X}}$ & the conditional density  of $Y$ conditioned on $X$	\\
      $p_{\mathsmaller{X,Y}}$ & the joint density of $X$ and $Y$\\
    $p_\mathsmaller{\epsilon|X}$ & the conditional density  of $\epsilon$ conditioned on $X$\\
    $f^\star$ & the underlying truth function in modal regression, see formula \eqref{data_generating_model}  \\
    $f_\mathsmaller{\mathrm{M}}	$ & the modal regression function or the conditional mode function, see formula \eqref{modal_function}     \\
    $f_{\mathbf{z},\sigma}$ & the empirical modal regression estimator  in $\mathcal{H}$, see formula \eqref{relaxation}  \\
    $f_{\mathcal{H},\sigma}$ & the data-free modal regression estimator in $\mathcal{H}$, see formula \eqref{target_function_model}  \\
    $f_\mathcal{H}$ & the data-free least squares regression estimator  in $\mathcal{H}$ \\
     $\mathcal{R}(f)$  & the modal regression risk for the hypothesis $f:\mathcal{X}\rightarrow \mathbb{R}$   \\
    $\mathcal{R}^\sigma(f)$ & the data-free generalization risk for the hypothesis $f:\mathcal{X}\rightarrow \mathbb{R}$  \\
     $\mathcal{R}_n^\sigma(f)$ & the empirical generalization risk for the hypothesis  $f:\mathcal{X}\rightarrow \mathbb{R}$\\
  \bottomrule
\end{tabular}
\caption{A list of notations and their definitions in this paper}\label{Table_properties}\label{table::notion}
\end{table}  

\subsection{Structure of This Paper} 
This paper is organized as follows: in Section \ref{sec::modal_regression_formulation}, we formulate the modal regression problem within the statistical learning framework. To this end, we introduce the modal regression function in Subsection \ref{modal_regression_setting}. We define the modal regression risk and characterize its \textit{Bayes} rule in Subsection \ref{subsec::Bayes_rule}. A kernel density estimation interpretation and an empirical risk minimization perspective of modal regression are provided in Subsections \ref{sub::mr_kde} and \ref{ERM_perspective}, respectively. Section \ref{sec::learning_theory} is devoted to developing a learning theory for modal regression. The modal regression calibration problem (see Subsection \ref{subsec::modal_regression_calibration}), the convergence of the excess generalization risk (see Subsection \ref{subsec::excess_generalization_risk}), and the function estimation calibration problem (see Subsection \ref{subsec::function_calibration}) are studied by applying standard learning theory arguments. Comparisons between our study and the existing ones are also mentioned in this section. In Section \ref{sec::correntropy}, we interpret MCCR from a modal regression viewpoint by suggesting that MCCR with a tending-to-zero scale parameter is essentially modal regression.  Since one of the main motivations of the present study is to understand MCCR within the statistical learning framework and having realized that MCCR with a tending-to-zero scale parameter is modal regression, we, therefore, retrospect MCCR in Section \ref{sec::back_to_MCCR} by applying the theory developed in Section \ref{sec::learning_theory} and depict a general picture of MCCR. Section \ref{sec::practical_issues} is concerned with the implementation issues in modal regression such as model selection and computational algorithms. Numerical validations will be provided in this section. We close this paper in Section \ref{sec::conclusion} with  conclusions. For the sake of readability, a list of notations and their definitions in this paper is provided in Table \ref{table::notion}.

\section{A Statistical Learning Framework for Modal Regression}\label{sec::modal_regression_formulation}

\subsection{Formulating the Modal Regression Problem}\label{modal_regression_setting}
We first formulate the modal regression problem formally in this subsection. To this end, we first assume that we are given a set of i.i.d observations $\mathbf{z}$ that are generated by 
\begin{align}\label{data_generating_model}
Y=f^\star(X)+\epsilon,
\end{align} 
where the mode of the conditional distribution of $\epsilon$ at any $x\in\mathcal{X}$ is assumed to be zero. That is, ${\sf{mode}}(\epsilon\,|\, X=x):=\arg\max_{t\in\mathbb{R}}p_{\mathsmaller{\epsilon|X}}(t\,|\,X=x)=0$ for any $x\in\mathcal{X}$, where $p_{\mathsmaller{\epsilon|X}}$ is the conditional density of $\epsilon$ conditioned on $X$. It is obvious from \eqref{data_generating_model} that under the zero-mode noise assumption, it holds that ${\sf{mode}}(Y|X)=f^\star(X)$. We further assume that $p_{\mathsmaller{\epsilon|X}}$ is continuous  and bounded on $\mathbb{R}$ for any $x\in\mathcal{X}$. Here, it should be remarked that in this study we do not assume either the homogeneity or the symmetry of the distribution of the noise $\epsilon$. In other words, the heterogeneity of the distribution of the residuals or the skewed noise distribution is allowed. 

In modal regression problems, we aim at approximating the \textit{modal regression function }(see formula (1.1) in \cite{collomb1987note}): 
\begin{definition}[Modal Regression Function] 
The \textbf{modal regression function} $f_\mathsmaller{\mathrm{M}}:\mathcal{X}\rightarrow\mathbb{R}$  is defined as
\begin{align}\label{modal_function}
f_\mathsmaller{\mathrm{M}}(x):=\arg\max_{t\in\mathbb{R}}p_{\mathsmaller{Y|X}}(t|\,X=x), \,\,x\in\mathcal{X},
\end{align}
where $p_{\mathsmaller{Y|X}}(\cdot|X)$ denotes the conditional density of $Y$ conditioned on $X$.
\end{definition}

Throughout this paper, we assume that the modal regression function $f_\mathsmaller{\mathrm{M}}$ is well-defined on $\mathcal{X}$. That is, $\arg\max_{t\in\mathbb{R}}p_{\mathsmaller{Y|X}}(t\mid\,X=x)$ is assumed to exist and be unique for any fixed $x\in\mathcal{X}$. Obviously, this is equivalent to assuming the existence and uniqueness of the global mode of the conditional density $p_{\mathsmaller{Y|X}}$. On the other hand, due to the zero-mode assumption of the conditional distribution of $\epsilon$ in \eqref{data_generating_model} for any $x\in\mathcal{X}$, we know that $f_\mathsmaller{\mathrm{M}} \equiv f^\star$. Consequently, the learning for modal regression problem is equivalent to the problem of learning the modal regression function $f_\mathsmaller{\mathrm{M}}$, and thus $f^\star$. Said differently, $f_\mathsmaller{\mathrm{M}}$ is the so-called target hypothesis.

From the definition, the modal regression function $f_\mathsmaller{\mathrm{M}}$ is defined as the maximum of the conditional density $p_{\mathsmaller{Y|X}}$ conditioned on $X$. Note that, maximizing the conditional density is equivalent to maximizing the joint density $p_\mathsmaller{X,Y}$ for any fixed realization of $X$. Therefore, it is direct to see that one can approximate $f_\mathsmaller{\mathrm{M}}$ by maximizing the conditional density $p_{\mathsmaller{Y|X}}$ or the joint density $p_\mathsmaller{X,Y}$, both of which can be estimated via kernel density estimation. This is, in fact, what most of the existing studies on modal regression do \citep[see e.g.,][]{collomb1987note,chen2016nonparametric,yao2016nonparametric}. However, estimating the conditional density $p_{\mathsmaller{Y|X}}$ or the joint density $p_\mathsmaller{X,Y}$ via kernel density estimation suffers from the \textit{curse of dimensionality} and is not feasible when the dimension of the input space is high. In this study, we are interested in an empirical risk minimization approach that is dimension-insensitive as formulated later.

\subsection{Modeling the Modal Regression Risk and Characterizing the \textit{Bayes} Rule}\label{subsec::Bayes_rule}
 
To be in a position to carry out a statistical learning assessment of modal regression, besides the target hypothesis defined above, we also need to devise a fitting risk that measures the goodness-of-fit when a candidate hypothesis is considered. The newly devised fitting risk should vote the target hypothesis \eqref{modal_function} as the best candidate when the hypothesis space is sufficiently large. This gives the main purpose of this subsection.

\begin{definition}[Modal Regression Risk]\label{modal_regression_risk}
For a  measurable function $f:\mathcal{X}\rightarrow \mathbb{R}$, its \textbf{modal regression risk} $\mathcal{R}(f)$ is defined as
\begin{align}\label{risk_functional_RF}
\mathcal{R}(f)=\int_{\mathcal{X}}p_{\mathsmaller{Y|X}}(f(x)|X=x)\mathrm{d}\rho_{\mathsmaller{\mathcal{X}}}(x).
\end{align}
\end{definition}

Analogously to learning for regression and classification scenarios \citep[see, e.g.,][]{cucker2007learning,steinwart2008support}, we denote the \textbf{\textit{Bayes} rule} of modal regression as the ``best" hypothesis favored by the above modal regression risk over the measurable function set $\mathcal{M}$ (comprised of all measurable functions from $\mathcal{X}$ to $\mathbb{R}$). The following conclusion indicates that the target hypothesis $f_\mathsmaller{\mathrm{M}}$ is exactly the \textit{Bayes} rule of modal regression.

\begin{theorem}\label{bayes_rule}
The modal regression function $f_\mathsmaller{\mathrm{M}}$ in \eqref{modal_function} gives the \textit{Bayes} rule of modal regression. That is, 
\begin{align*}
f_\mathsmaller{\mathrm{M}}=\arg\max_{f\in\mathcal{M}}\mathcal{R}(f).
\end{align*}
\end{theorem}
\begin{proof} 
Recall that the conditional mode function $f_\mathsmaller{\mathrm{M}}$ is given as
\begin{align*} 
f_\mathsmaller{\mathrm{M}}(x)=\arg\max_{t\in\mathbb{R}}p_{\mathsmaller{Y|X}}(t\,|\,X=x), \,\,x\in\mathcal{X}.
\end{align*}
Following the modal regression risk defined in Definition \ref{modal_regression_risk}, for any measurable function $f\in\mathcal{M}$, we have
\begin{align*}
\mathcal{R}(f)=\int_{\mathcal{X}}p_{\mathsmaller{Y|X}}	(f(x)|X=x)\mathrm{d}\rho_{\mathsmaller{\mathcal{X}}}(x)\leq \int_{\mathcal{X}}  p_{\mathsmaller{Y|X}}	(f_\mathsmaller{\mathrm{M}}(x)|X=x) \mathrm{d}\rho_{\mathsmaller{\mathcal{X}}}(x)=\mathcal{R}(f_\mathsmaller{\mathrm{M}}),
\end{align*}
which directly yields
\begin{align*}
f_\mathsmaller{\mathrm{M}}=\arg\max_{f\in\mathcal{M}}\mathcal{R}(f).
\end{align*}
This completes the proof of Theorem \ref{bayes_rule}.
\end{proof}

The plausibility of the above-defined modal regression risk stems from the fact that $f_\mathsmaller{\mathrm{M}}$ is the \textit{Bayes} rule of modal regression, as justified by Theorem \ref{bayes_rule}. With the modal regression risk being defined and recalling that $f_\mathsmaller{\mathrm{M}}$ maximizes the modal regression risk, the most direct way to learn $f_\mathsmaller{\mathrm{M}}$ is to maximize the sample analogy of the modal regression risk. Unfortunately, this is intractable since the discretization of an unknown conditional density is involved. In the next subsection, to circumvent this problem, we introduce a surrogate of the modal regression risk.

\begin{remark}
We now give a remark on the terminology ``risk". For any measurable function $f:\mathcal{X}\rightarrow \mathbb{R}$, the modal regression risk $\mathcal{R}(f)$ in Definition \ref{modal_regression_risk} can be regarded as a measure of the extent to which the function $f$ fits the \textit{Bayes} rule $f_\mathsmaller{\mathrm{M}}$ in the $\mathcal{R}(\cdot)$ sense. Therefore, the terminology ``risk" is not used as what is commonly referred to in the statistical learning literature. However, in what follows, given the one-to-one correspondence between the corresponding maximization and minimization problems, we still term  $\mathcal{R}(f)$ as the (modal regression) risk of $f$.  
\end{remark}

\subsection{Learning for Modal Regression via Kernel Density Estimation}\label{sub::mr_kde}
We now show that the modal regression problem can be tackled by applying the kernel density estimation technique. To this purpose, let $f: \mathcal{X}\rightarrow \mathbb{R}$ be a measurable function and denote $E_f$ as the random variable induced by the residual $Y-f(X)$, where the subscript $f$ indicates its dependence on $f$. We also denote $p_\mathsmaller{E_f}$, or simply $p_\mathsmaller{f}$, as the density function of the random variable $E_f$ and denote $p_{\epsilon|X}$ as the conditional density of the random variable $\epsilon = Y-f^\star(X)$. The following theorem, which was first established in \cite{fan2014consistency}, relates the modal regression risk of $f$ to $p_{\epsilon|X}$ and $p_\mathsmaller{E_f}$.   

\begin{theorem}\label{density_change_variable}
Let $f: \mathcal{X}\rightarrow \mathbb{R}$ be a measurable function. Then,  
\begin{align*}
\int_\mathcal{X}p_{\epsilon|X}(\cdot + f(x) - f^\star(x)|X=x)\mathrm{d}\rho_{\mathsmaller{\mathcal{X}}}(x) 
\end{align*}
is a density of the random variable $E_f:=Y-f(X)$, which is denoted as $p_\mathsmaller{E_f}$. Correspondingly, we have $p_\mathsmaller{E_f}(0)=\mathcal{R}(f)$.
\end{theorem}
\begin{proof} 
From the model assumption that $\epsilon = Y - f^\star(X)$, we have
\begin{align*}
\epsilon = E_f + f(X) - f^\star(X).
\end{align*}
As a result, the density function of the error variable $E_f$ can be expressed as
\begin{align*}
\int_\mathcal{X}p_{\epsilon|X}(\cdot + f(x) - f^\star(x)|X=x)\mathrm{d}\rho_{\mathsmaller{\mathcal{X}}}(x) 
\end{align*}
and denoted by $p_\mathsmaller{E_f}$. Moreover, from the definition of the risk functional $\mathcal{R}(\cdot)$ in \eqref{risk_functional_RF}, we know that  
\begin{align*}
p_\mathsmaller{E_f}(0)&=\int_{\mathcal{X}}p_{\epsilon|X}( f(x) - f^\star(x)|X=x) \mathrm{d}\rho_{\mathsmaller{\mathcal{X}}}(x)\\
&= \int_{\mathcal{X}}p_{\mathsmaller{Y|X}}(f(x)|X=x)\mathrm{d}\rho_{\mathsmaller{\mathcal{X}}}(x)\\
&= \mathcal{R}(f).
\end{align*}
This completes the proof of Theorem \ref{density_change_variable}.
\end{proof}

From Theorem \ref{density_change_variable}, the hypothesis $f$ that maximizes the modal regression risk $\mathcal{R}(f)$ is the one that maximizes the density of $E_f:=Y-f(X)$ at $0$, which can be estimated non-parametrically. In this study, the kernel density estimation technique is tailored to modal regression with the help of the modal regression kernel defined below. 
 
\begin{definition}[Modal Regression Kernel]\label{modal_regression_kernel}
A kernel $K_\sigma:\mathbb{R}\times\mathbb{R}\rightarrow \mathbb{R}_+$ is said to be a \textbf{modal regression kernel} with the representing function $\phi$ and the bandwidth parameter $\sigma>0$ if there exists a function $\phi:\mathbb{R}\rightarrow [0,\infty)$ such that $K_\sigma(u_1,u_2)=\phi\left(\frac{u_1-u_2}{\sigma}\right)$ for any $u_1,u_2\in\mathbb{R}$, $\phi(u)=\phi(-u)$, $\phi(u)\leq \phi(0)$ for any $u\in\mathbb{R}$, and $\int_\mathbb{R}\phi(u)\mathrm{d}u=1$. 
\end{definition}

According to Definition \ref{modal_regression_kernel}, it is easy to see that common smoothing kernels \citep[see, e.g.,][]{wand1994kernel}   such as the Naive kernel, the Gaussian kernel, the Epanechnikov kernel, and the Triangular kernel are modal regression kernels. Their corresponding representing functions can be easily deduced with simple computations. For a modal regression kernel $K_\sigma$ with the representing function $\phi$, throughout this paper, without loss of generality, we assume $\phi(0)=1$. 

As a consequence of Theorem \ref{density_change_variable}, for any measurable function $f$, we know that $p_\mathsmaller{f}(0)=\mathcal{R}(f)$. With the help of a modal regression kernel $K_\sigma$, it is immediate to see that an empirical kernel density estimator $\hat{p}_\mathsmaller{f}$ for $p_\mathsmaller{f}$ at $0$ can be formulated as follows 
\begin{align*}
\hat{p}_\mathsmaller{f}(0)=\frac{1}{n\sigma}\sum_{i=1}^n K_\sigma (y_i-f(x_i),0)=\frac{1}{n\sigma}\sum_{i=1}^n K_\sigma (y_i, f(x_i)):=\mathcal{R}_n^\sigma(f).
\end{align*}
Therefore, when confined to a hypothesis space $\mathcal{H}$, learning a function $f$ that maximizes the modal regression risk is cast as learning the function $f$ that maximizes the value of the empirical density estimator $\hat{p}_\mathsmaller{f}$ at $0$. Thus, the empirical target hypothesis is modeled as    
\begin{align}\label{relaxation}
\begin{split}
f_{\mathbf{z},\sigma}:&=\arg\max_{f\in\mathcal{H}}\hat{p}_\mathsmaller{f}(0)\\
&=\arg\max_{f\in\mathcal{H}}\mathcal{R}_n^\sigma(f),
\end{split}
\end{align}
where $\mathcal{H}$ is assumed to be a compact subset of $C(\mathcal{X})$ throughout this paper. The population version of $f_{\mathbf{z},\sigma}$ can be expressed as
\begin{align}\label{target_function_model}
f_{\mathcal{H},\sigma}:=\arg\max_{f\in\mathcal{H}}\mathcal{R}^\sigma(f),
\end{align}
where $\mathcal{R}^\sigma(\cdot)$ is the expectation of $\mathcal{R}_n^\sigma(\cdot)$ with respect to the random samples $\mathbf{z}$ and for any $f:\mathcal{X}\rightarrow\mathbb{R}$, it can be expressed as 
\begin{align*}
\mathcal{R}^\sigma(f)= \frac{1}{\sigma}\int_{\mathcal{X}\times\mathcal{Y}}\phi\left(\frac{y-f(x)}{\sigma}\right)\mathrm{d}\rho(x,y).	
\end{align*}

The risk functional $\mathcal{R}^\sigma(f)$ defined above gives the \textbf{generalization risk} of $f$ when a modal regression kernel $K_\sigma$ with the representing function $\phi$ is adopted. As we shall see later, it can be seen as a surrogate of the true modal regression risk $\mathcal{R}(f)$ since $\mathcal{R}^\sigma(f)$ approximates $\mathcal{R}(f)$ when $\sigma\rightarrow 0$. The interpretation of modal regression from a kernel density estimation viewpoint explains the requirement that $\int_{\mathbb{R}}\phi(u)\mathrm{d}u=1$ in Definition \ref{modal_regression_kernel}.

\subsection{Modal Regression: an Empirical Risk Minimization View}\label{ERM_perspective}
In the preceding subsection, we showed that the modal regression scheme \eqref{relaxation} can be interpreted from a kernel density estimation point of view. Maximizing the value of the kernel density estimator for $E_f$ at $0$ encourages the considered hypothesis $f$ to approximate the projection of the \textit{Bayes} rule onto $\mathcal{H}$, i.e., $f_{\mathcal{H},\sigma}$. In this subsection, we show that one can also interpret the modal regression scheme \eqref{relaxation} by using the language of empirical risk minimization.   

To proceed, let us consider a modal regression kernel $K_\sigma$ with the representing function $\phi$ and the scale parameter $\sigma>0$. We then introduce the following distance-based modal regression loss $\phi_\sigma:\mathbb{R}\rightarrow [0,\infty)$:
\begin{align}\label{introduced_loss}
\phi_\sigma(y-f(x))=\sigma^{-1}\left(1-\phi\left((y-f(x))\sigma^{-1}\right)\right).
\end{align}
Based on the newly introduced loss $\phi_\sigma$, the modal regression scheme \eqref{relaxation} can be reformulated as follows
\begin{align}\label{ERM_empirical_target}
f_{\mathbf{z},\sigma}=\arg\min_{f\in\mathcal{H}}\frac{1}{n}\sum_{i=1}^n \phi_\sigma(y_i,f(x_i)),
\end{align}
and, similarly, its data-free counterpart can be formulated as
\begin{align}\label{ERM_expected_target}
f_{\mathcal{H},\sigma}=\arg\min_{f\in\mathcal{H}}\int_{\mathcal{X}\times\mathcal{Y}} \phi_\sigma(y,f(x))\mathrm{d}\rho.
\end{align}
It is easy to see that the empirical estimator \eqref{ERM_empirical_target} is an M-estimator and the two formulations of $f_{\mathbf{z},\sigma}$ in \eqref{relaxation} and \eqref{ERM_empirical_target}  are, in fact, equivalent. Similarly, one also obtains the same target hypothesis from  \eqref{target_function_model} and \eqref{ERM_expected_target}.

\begin{remark}
For formulation simplification, whenever referred to herein, $f_{\mathbf{z},\sigma}$ and $f_{\mathcal{H},\sigma}$ will be pointed to the estimators formulated by \eqref{relaxation} and \eqref{target_function_model}, respectively, while keeping in mind that the conducted analysis on $f_{\mathbf{z},\sigma}$ is inspired by and within the ERM framework. 
\end{remark}

\section{A Learning Theory of Modal Regression}\label{sec::learning_theory}
In this section, we aim to develop a learning theory for modal regression which can be used to assess the statistical learning performance of the modal regression estimator $f_{\mathbf{z},\sigma}$.

\subsection{Learning the Conditional Mode: Three Building Blocks}\label{Subsection::three_blocks}
In Section \ref{sec::modal_regression_formulation},  for a given hypothesis $f$, the modal regression risk $\mathcal{R}(f)$ is defined; moreover, it turns out that $f_{\mathsmaller{\mathrm{M}}}$ is the \textit{Bayes} rule of modal regression. On the other hand, we show that the modal regression estimator can be learned via maximizing the risk functional $\mathcal{R}^\sigma_n(\cdot)$. Recalling that the central concern in learning theory is risk consistency under various notions and following the clue of existing learning theory studies on the binary-classification problem, it is natural and necessary to investigate the following three problems:  
\begin{enumerate}
\item The problem of the excess generalization risk consistency and convergence rates, i.e., the convergence from $\mathcal{R}^\sigma(f_{\mathbf{z},\sigma})$ to $\mathcal{R}^\sigma(f^\star)$.
\item The modal regression calibration problem, i.e., whether the convergence from  $\mathcal{R}^\sigma(f_{\mathbf{z},\sigma})$ to $\mathcal{R}^\sigma(f^\star)$ implies the convergence from $\mathcal{R}(f_{\mathbf{z},\sigma})$ to $\mathcal{R}(f^\star)$? 
\item The function estimation calibration problem, i.e., whether the convergence from $\mathcal{R}(f_{\mathbf{z},\sigma})$ to $\mathcal{R}(f^\star)$ implies the convergence from $f_{\mathbf{z},\sigma}$ to $f^\star$?
\end{enumerate}

\begin{figure}[h]
\begin{center}
\begin{tikzpicture}
\tikzset{trim left=0.0 cm}
\draw[fill=blue!0!white, solid] (0,0) rectangle (3.3,1.2)node[pos=.5] {$f_{\mathbf{z},\sigma}\rightarrow f^\star$};
\draw[fill=blue!0!white, solid] (5,0) rectangle (8.5,1.2)node[pos=.5] {$\mathcal{R}(f_{\mathbf{z},\sigma})\rightarrow\mathcal{R}(f^\star)$};
\draw[fill=blue!0!white, solid] (10,0) rectangle (13.5,1.2)node[pos=.5] {$\mathcal{R}^\sigma(f_{\mathbf{z},\sigma})\rightarrow\mathcal{R}^\sigma(f^\star)$};
\path[draw=black,solid,line width=2mm,fill=blue, preaction={-triangle 90,thin,draw,shorten >=-1mm}
]  (4.9,0.6) --(3.4,0.6);
\path[draw=black,solid,line width=2mm,fill=blue, preaction={-triangle 90,thin,draw,shorten >=-1mm} 
]  (9.9,0.6) --(8.6,0.6);
\end{tikzpicture}
\end{center}
\caption{An illustration of the three building blocks in learning for modal regression. The left block stands for the function estimation consistency of $f_{\mathbf{z},\sigma}$, the middle block denotes the modal regression consistency of $f_{\mathbf{z},\sigma}$, while the right block represents the excess generalization risk consistency of $f_{\mathbf{z},\sigma}$.}\label{block_illustration}
\end{figure} 

The above three problems are fundamental in conducting a learning theory analysis on modal regression and serve as three main building blocks. Detailed explorations will be expanded in the following subsections.

\subsection{Towards the Modal Regression Calibration Problem}\label{subsec::modal_regression_calibration}
We first investigate the modal regression calibration problem stated in Question $1$, i.e., whether the convergence from  $\mathcal{R}^\sigma(f_{\mathbf{z},\sigma})$ to $\mathcal{R}^\sigma(f^\star)$ implies the convergence from $\mathcal{R}(f_{\mathbf{z},\sigma})$ to $\mathcal{R}(f^\star)$. To this end, we need to confine ourselves to the calibrated modal regression kernel defined below.

\begin{definition}[Calibrated Modal Regression Kernel]\label{calibrated_kernel}
A modal regression kernel $K_\sigma$ with the representing function $\phi$ is said to be a \textbf{calibrated modal regression kernel} if it satisfies the following conditions:
\begin{itemize}
\item [$(i)$]  $\phi$ is bounded;
\item [$(ii)$] $\phi$ is Lipschitz continuous on $\mathbb{R}$ with the Lipschitz constant $L$;
\item [$(iii)$] $\int_\mathbb{R}u^2\phi(u)\mathrm{d}u<\infty$.
\end{itemize}
\end{definition}

Another restriction we need to impose is on the conditional density $p_{\epsilon|X}$ as follows:

\begin{assumption}\label{smooth_density}
The conditional density of $\epsilon$ given $X$, namely, $p_{\epsilon|X}$, is second-order continuously differentiable and $\|p_{\epsilon|X}^{\prime\prime}\|_\infty$ is bounded from above.
\end{assumption}
 
\begin{theorem}\label{modal_regression_calibration}
Suppose that Assumption \ref{smooth_density} holds and let $K_\sigma$ be a calibrated modal regression kernel with the representing function $\phi$ and the scale parameter $\sigma$. For any measurable function $f:\mathcal{X}\rightarrow\mathbb{R}$, it holds that
\begin{align*}
\Big|\{\mathcal{R}(f^\star)-\mathcal{R}(f)\}-\{\mathcal{R}^\sigma(f^\star)-\mathcal{R}^\sigma(f)\}\Big|\leq c_1 \sigma^2,
\end{align*}
where $c_1= \|p^{\prime\prime}_{\epsilon|\mathsmaller{X}}\|_\infty\int_\mathbb{R} u^2\phi(u)\mathrm{d}u$.
\end{theorem} 
\begin{proof} 
Recalling the definition of the risk functional $\mathcal{R}^\sigma(f)$ for any measurable function $f:\mathcal{X}\rightarrow\mathbb{R}$ and applying Taylor's Theorem to the conditional density $p_{\epsilon|X}$, we have
\begin{align}\label{eqaulities}
\begin{split}
\mathcal{R}^\sigma(f) &= \frac{1}{\sigma}\int_{\mathcal{X}\times\mathcal{Y}}\phi\left(\frac{y -f(x)}{\sigma}\right)\mathrm{d}\rho(x,y)\\
&= \frac{1}{\sigma}\int_\mathcal{X} \int_{\mathbb{R}}\phi\left(\frac{t - (f(x) - f^\star(x))}{\sigma}\right)p_{\epsilon|X}(t\mid X=x)\mathrm{d}t\,\mathrm{d}\rho_{\mathsmaller{\mathcal{X}}}(x)\\
&= \int_\mathcal{X}\int_{\mathbb{R}}\phi(u)p_{\epsilon|X}(f(x) - f^\star(x) + \sigma u\mid X=x) \mathrm{d}u \mathrm{d}\rho_{\mathsmaller{\mathcal{X}}}(x)\\
&= \int_\mathcal{X}\int_{\mathbb{R}}\phi(u)p_{\epsilon|X}(f(x) - f^\star(x)\mid X=x) \mathrm{d}u \mathrm{d}\rho_{\mathsmaller{\mathcal{X}}}(x)\\
&\phantom{=} + \sigma \int_{\mathcal{X}}\int_{\mathbb{R}}u\phi(u)p_{\epsilon|X}^\prime(f(x) - f^\star(x)\mid X=x)\mathrm{d}u \mathrm{d}\rho_{\mathsmaller{\mathcal{X}}}(x) \\
&\phantom{=} + \frac{\sigma^2}{2} \int_{\mathcal{X}}\int_{\mathbb{R}}u^2 \phi(u)p_{\epsilon|X}^{\prime\prime}(\eta_x \,|\, X=x)\mathrm{d}u \mathrm{d}\rho_{\mathsmaller{\mathcal{X}}}(x),
\end{split} 
\end{align}
where, for any fixed $x\in\mathcal{X}$, the point $\eta_x$ lies between $f(x) - f^\star(x)$ and $f(x) - f^\star(x)+\sigma u$.

The fact that $K_\sigma$ is a calibrated modal regression kernel with the representing function $\phi$ ensures $\int_\mathbb{R}\phi(u)\mathrm{d}u=1$ and reminds the symmetry of $\phi$ on $\mathbb{R}$, which further indicates that $\int_\mathbb{R} u\phi(u)\mathrm{d}u=0$. On the other hand, the fact that 
\begin{align*}
\mathcal{R}(f) = \int_\mathcal{X} p_\mathsmaller{{\epsilon|X}}(f(x) - f^\star(x)|X=x) \mathrm{d}\rho_{\mathsmaller{\mathcal{X}}}(x),
\end{align*}
together with Equalities \eqref{eqaulities} yields
\begin{align*}
\left|\mathcal{R}^\sigma(f) - \mathcal{R}(f)\right| \leq \frac{\sigma^2}{2}\left(\|p''_{\epsilon|X}\|_\infty \int_\mathbb{R} u^2\phi(u)\mathrm{d}u \right).
\end{align*}
Denoting $c_1:=\|p''_{\epsilon|X}\|_\infty \int_\mathbb{R} u^2\phi(u)\mathrm{d}u$, we accomplish the proof of Theorem \ref{modal_regression_calibration}.
\end{proof}

\begin{remark}
The proof of Theorem \ref{modal_regression_calibration} indicates that $\mathcal{R}^\sigma(f)$ is a second-order approximation (with respect to $\sigma$) of $\mathcal{R}(f)$ since $\mathcal{R}^\sigma(f)-\mathcal{R}(f)=\mathcal{O}(\sigma^2)$. In fact, if a higher-order kernel \citep[see e.g., Section 2.8 in][]{wand1994kernel} is used, a higher-order approximation of $\mathcal{R}(f)$ can be expected. 
\end{remark}
 
From the proof of Theorem \ref{modal_regression_calibration}, we see that when $K_\sigma$ is a calibrated modal regression kernel with the representing function $\phi$ and the scale parameter $\sigma$,  for any measurable function $f:\mathcal{X}\rightarrow \mathbb{R}$, the generalization risk $\mathcal{R}^\sigma(f)$ approaches the true modal regression risk $\mathcal{R}(f)$ provided that $\sigma\rightarrow 0$. Therefore, in the above sense,  $\mathcal{R}^\sigma(f)$ can be considered as a relaxation of  $\mathcal{R}(f)$. On the other hand, Theorem \ref{modal_regression_calibration} indicates that the difference between the \textbf{excess modal regression risk} $\mathcal{R}(f^\star)-\mathcal{R}(f)$ and the \textbf{excess generalization risk} $\mathcal{R}^\sigma(f^\star)-\mathcal{R}^\sigma(f)$ can be upper bounded by $\mathcal{O}(\sigma^2)$. Clearly, under the assumptions of Theorem \ref{modal_regression_calibration}, when $\sigma\rightarrow 0$, $\mathcal{R}^\sigma(f^\star)-\mathcal{R}^\sigma(f)$ also approaches $\mathcal{R}(f^\star)-\mathcal{R}(f)$. In this sense, Theorem \ref{modal_regression_calibration} establishes a  \textit{comparison theorem} akin to the one in the classification scenario \citep[see][]{zhang2004statistical,bartlett2006convexity}. This elucidates the terminology---the calibrated modal regression kernel, and the terminology---the modal regression calibration problem.

\subsection{Towards the Convergence Rates of the Excess Generalization Risk}\label{subsec::excess_generalization_risk}
One of the main focuses in learning theory is the generalization ability of a learning algorithm that measures its out-of-sample prediction ability. It plays an important role in designing learning algorithms with theoretical guarantees.  In this subsection, we derive the generalization bounds for the modal regression estimator $f_{\mathbf{z},\sigma}$, i.e., the convergence rates of $\mathcal{R}^\sigma(f^\star) - \mathcal{R}^\sigma(f_{\mathbf{z},\sigma})$, by means of learning theory arguments. The following assumption is needed for this purpose: 

\begin{assumption}\label{assum:bound_capacity}
We make the following assumptions:
\begin{itemize}
\item [(i)] There exists a positive constant $M$ such that $\|f^\star\|_\infty \leq M$;
\item [(ii)] $\sup_{t\in\mathbb{R},\,x\in\mathcal{X}}p_{\epsilon|X}(t\mid X=x) = c_2 <\infty$;
\item [(iii)] For any $\varepsilon>0$,  there exists an exponent $p$ with $0<p<2$ such that the $\ell^2$-empirical covering number (with radius $\varepsilon$) of $\mathcal{H}$, denoted as $\mathcal{N}_{2,\mathbf{x}}(\mathcal{H}, \varepsilon)$, satisfies
\begin{align*}
\log\mathcal{N}_{2,\mathbf{x}}(\mathcal{H}, \varepsilon) \lesssim  \varepsilon^{-p},
\end{align*}
where the definition of the empirical covering number is provided below (see also \cite{anthony2009neural}), and the notation $a\lesssim b$ for $a,b\in\mathbb{R}$ means that there exists a positive constant $c$ such that $a\leq c b$. 	  
\end{itemize}

\end{assumption}

\begin{definition}[$\ell^2$-empirical Covering Number]
Let $\mathcal{F}$ be a set of functions on $\mathcal{X}$ and $\mathbf{x}=\{x_1,\cdots,x_m\}\subset \mathcal{X}$. The metric $d_{2,\mathbf{z}}$ is defined on $\mathcal{F}$ by 
\begin{align*}
d_{2,\mathbf{x}}(f,g)=\left\{\frac{1}{m}\sum_{i=1}^m(f(x_i)-g(x_i))^2\right\}^{1/2}.
\end{align*}
For every $\varepsilon>0$, the \textbf{$\ell^2$-empirical covering number} of $\mathcal{F}$ with respect to $d_{2,\mathbf{x}}$ is defined as
\begin{align*}
\mathcal{N}_{2,\mathbf{x}}(\mathcal{F},\varepsilon)=\inf\left\{\ell\in\mathbb{N}:\exists \{f_i\}_{i=1}^\ell \,\,\hbox{such that}\,\, \mathcal{F}=\cup_{i=1}^\ell\{f\in\mathcal{F}:d_{2,\mathbf{x}}(f,f_i)\leq \varepsilon\}\right\}.
\end{align*}
\end{definition}

Restrictions in Assumption \ref{assum:bound_capacity} are fairly standard if we recall that the hypothesis space $\mathcal{H}$ is assumed to be a compact subset of $C(\mathcal{X})$. In what follows, without loss of generality, we also assume that $\|f\|_\infty \leq M$ for any $f\in\mathcal{H}$. The following error decomposition lemma is helpful in bounding the excess generalization error.

\begin{lemma}\label{lem:decompose}
Let $f_{\mathbf{z},\sigma}$ be produced by \eqref{relaxation} and assume that $f^\star\in\mathcal{H}$. Then we have
\begin{align*}
\mathcal{R}^\sigma(f^\star) - \mathcal{R}^\sigma(f_{\mathbf{z},\sigma}) \leq 
\mathcal{R}^\sigma(f_{\mathcal{H},\sigma})- \mathcal{R}_n^\sigma(f_{\mathcal{H},\sigma}) +
 \mathcal{R}^\sigma_n(f_{\mathbf{z},\sigma}) - \mathcal{R}^\sigma(f_{\mathbf{z},\sigma}).
\end{align*}
\end{lemma}
\begin{proof} 
Recalling that $f_{\mathcal{H},\sigma}=\arg\max_{f\in\mathcal{H}}\mathcal{R}^\sigma(f)$, we have
\begin{align*}
\begin{split}
\mathcal{R}^\sigma(f^\star) - \mathcal{R}^\sigma(f_{\mathbf{z},\sigma}) & \leq \mathcal{R}^\sigma(f_{\mathcal{H},\sigma}) - \mathcal{R}^\sigma(f_{\mathbf{z},\sigma}) \\ &\leq \mathcal{R}^\sigma(f_{\mathcal{H},\sigma})  - \mathcal{R}^\sigma_n(f_{\mathcal{H},\sigma}) + \mathcal{R}^\sigma_n(f_{\mathcal{H},\sigma}) - \mathcal{R}^\sigma_n(f_{\mathbf{z},\sigma})+  \mathcal{R}^\sigma_n(f_{\mathbf{z},\sigma}) - \mathcal{R}^\sigma(f_{\mathbf{z},\sigma})\\
&\leq  \mathcal{R}^\sigma(f_{\mathcal{H},\sigma})- \mathcal{R}_n^\sigma(f_{\mathcal{H},\sigma}) +
 \mathcal{R}^\sigma_n(f_{\mathbf{z},\sigma}) - \mathcal{R}^\sigma(f_{\mathbf{z},\sigma}),
\end{split}
\end{align*}
where the last inequality is due to the fact that the quantity $\mathcal{R}^\sigma_n(f_{\mathcal{H},\sigma}) - \mathcal{R}^\sigma_n(f_{\mathbf{z},\sigma})$ is at most zero. This completes the proof of Lemma \ref{lem:decompose}.
\end{proof}

The following lemma, established in \cite{wu2007multi}, provides a Bernstein-type concentration inequality for function-valued random variables. It was proved by applying the local Rademacher complexity arguments developed in \cite{bartlett2005local}.   
\begin{lemma}\label{lem:concentration}
Let $\mathcal{F}$ be a class of bounded measurable functions. Assume that there are constants $\gamma\in[0, 1]$ and $B, c_\gamma >0$ such that $\|f\|_\infty\leq B$ and $\mathbb{E}f^2 \leq c_\gamma(\mathbb{E}f)^\gamma$ for every $f\in\mathcal{F}$. If for some $c_p > 0$ and $0<p<2$,
\[\log\mathcal{N}_{2,\mathbf{x}}(\mathcal{F}, \varepsilon) \leq c_p \varepsilon^{-p}, \,\,\forall \varepsilon > 0,\]
then there exists a constant $c'_p$ depending only on $p$ such that for any $t > 0$, with probability at least $1-e^{-t}$, it holds that
\[\mathbb{E}f - \frac1n\sum_{i=1}^{n}f(z_i) \leq \frac12\eta^{1-\gamma}(\mathbb{E}f)^\gamma + c'_p\eta + 2\left(\frac{c_\gamma t}{n}\right)^{\frac1{2-\gamma}}+\frac{18Bt}{n},\,\, \forall f\in\mathcal{F},\]
where
\[\eta = \max\left\{c_\gamma^{\frac{2-p}{4-2\gamma+p\gamma}}\left(\frac{c_p}{n}\right)^{\frac2{4-2\gamma+p\gamma}}, B^{\frac{2-p}{2+p}}\left(\frac{c_p}{n}\right)^{\frac2{2+p}}\right\}. \]
\end{lemma}

\begin{theorem}\label{generalization_bounds_intermediate}
Suppose that Assumption \ref{assum:bound_capacity} holds, $f^\star\in\mathcal{H}$, and the risk functional $\mathcal{R}^\sigma(\cdot)$ is defined in association with a calibrated modal regression kernel $K_\sigma$ and the representing function $\phi$. Let $f_{\mathbf{z},\sigma}$ be produced by \eqref{relaxation} with $\sigma\leq 1$. Then for any $0<\delta<1$, with probability at least $1-\delta$, it holds that
\begin{align*}
\mathcal{R}^\sigma(f^\star) - \mathcal{R}^\sigma(f_{\mathbf{z},\sigma}) \lesssim  \left(\frac{1}{n\sigma}+\frac{\sigma^{-\frac{2+3p}{4}}}{n^{1/2}}+\frac{\sigma^{-\frac{2+3p}{2+p}}}{n^{\frac{2}{2+p}}}\right)\log\left(\frac{1}{\delta}\right).  
\end{align*}
\end{theorem}	
\begin{proof} 
We prove the theorem by applying Lemma \ref{lem:concentration} to the following function-valued random variable on $\mathcal{Z}=\mathcal{X}\times\mathcal{Y}$:
\begin{align}\label{RV}
\xi(z) :=   \frac{1}{\sigma}\phi \left(\frac{y - f_{\mathcal{H},\sigma}(x)}{\sigma} \right) -  \frac{1}{\sigma}\phi \left(\frac{y - f(x)}{\sigma} \right),
\end{align}
where $f_{\mathcal{H},\sigma}$ is given in \eqref{target_function_model} and $f\in\mathcal{H}$. Due to the boundedness assumption of $\phi$, it is easy to see that $|\xi(z)| \leq  2\|\phi\|_\infty/\sigma $. Moreover, recalling the definition of the risk functional $\mathcal{R}^\sigma(\cdot)$, the following inequality holds
\begin{align}\label{variance}
\begin{split}
\mathbb{E}\xi^2  & =  \mathbb{E}\left[\frac{1}{\sigma}\phi \left(\frac{Y - f_{\mathcal{H},\sigma}(X)}{\sigma} \right) -  \frac{1}{\sigma}\phi \left(\frac{Y - f(X)}{\sigma} \right) \right]^2 \\
&\leq  \frac{2\|\phi\|_\infty}{\sigma}(\mathcal{R}^\sigma(f_{\mathcal{H},\sigma})+\mathcal{R}^\sigma(f)).
\end{split} 
\end{align}
From the proof of Theorem \ref{modal_regression_calibration}, we know that 
\begin{align*}
\mathcal{R}^\sigma(f_{\mathcal{H},\sigma})\leq \mathcal{R}(f_{\mathcal{H},\sigma}) + \frac{\sigma^2}{2}\left(\|p''_{\epsilon|X}\|_\infty \int_\mathbb{R} u^2\phi(u)\mathrm{d}u \right).
	\end{align*}
Similarly, we also have
\begin{align*}
\mathcal{R}^\sigma(f)\leq \mathcal{R}(f) + \frac{\sigma^2}{2}\left(\|p''_{\epsilon|X}\|_\infty \int_\mathbb{R} u^2\phi(u)\mathrm{d}u \right).
\end{align*}
The above two inequalities together with the bound for $\mathbb{E}\xi^2$ and the fact that $\sigma\leq 1$ yield   
\begin{align*}
\mathbb{E}\xi^2 & \leq \frac{2\|\phi\|_\infty}{\sigma}(\mathcal{R}(f_{\mathcal{H},\sigma})+\mathcal{R}(f)+c_1\sigma^2)\\
& \leq \frac{2\|\phi\|_\infty}{\sigma}(p_\mathsmaller{f_\mathsmaller{\mathcal{H},\sigma}}(0)+p_\mathsmaller{f}(0)+c_1\sigma^2)\\
& \lesssim \sigma^{-1},
\end{align*}
where the last inequality is due to the boundedness assumption of the conditional density of $\epsilon$ while the second inequality is a consequence of Theorem \ref{density_change_variable}.

Recalling that $\phi$ is Lipschitz continuous on $\mathbb{R}$ with the Lipschitz constant $L$, for any $f_1,f_2\in\mathcal{H}$, we thus have
\begin{align*}
\left|\frac{1}{\sigma}\phi\left(\frac{y-f_1(x)}{\sigma}\right)-\frac{1}{\sigma}\phi\left(\frac{y-f_2(x)}{\sigma}\right)\right|\leq \frac{L}{\sigma^2}\|f_1-f_2\|_\infty.   
\end{align*}
Consequently, if we denote $\mathcal{F}_\mathcal{H}$ as the following set
\begin{align*}
\mathcal{F}_\mathcal{H}:=\left\{g\,\Big|\,g(z)=  \frac{1}{\sigma}\phi \left(\frac{y - f_{\mathcal{H},\sigma}(x)}{\sigma} \right) -  \frac{1}{\sigma}\phi \left(\frac{y - f(x)}{\sigma} \right), f\in\mathcal{H}\right\},    
\end{align*}
then Assumption \ref{assum:bound_capacity} (iii) implies that  
\begin{align*}
\log \mathcal{N}_{2,\mathbf{x}}(\mathcal{F}_\mathcal{H},\varepsilon)\leq \log\mathcal{N}_{2,\mathbf{x}}(\mathcal{H},\varepsilon \sigma^2/L) \lesssim (\varepsilon \sigma^2)^{-p}.    
\end{align*}
Applying Lemma \ref{lem:concentration} to the random variable $\xi$ with $B= 2\|\phi\|_\infty/\sigma$, $\gamma=0$, $c_p=\sigma^{-2p}$, and $c_\gamma=\sigma^{-1}$, then for any $0<\delta<1$, with probability at least $1-\delta$, it holds that
\begin{align*}
\mathcal{R}^\sigma(f_{\mathcal{H},\sigma}) - \mathcal{R}^\sigma(f) - (\mathcal{R}_n^\sigma(f_{\mathcal{H},\sigma}) - \mathcal{R}_n^\sigma(f))\lesssim   \left(\frac{1}{n\sigma}+\frac{\sigma^{-\frac{2+3p}{4}}}{n^{1/2}}+\frac{\sigma^{-\frac{2+3p}{2+p}}}{n^{\frac{2}{2+p}}}\right)\log\left(\frac{1}{\delta}\right). 
\end{align*}
Noticing that the above inequality holds for any $f\in\mathcal{H}$ and recalling Lemma \ref{lem:decompose}, we obtain the desired conclusion in Theorem \ref{generalization_bounds_intermediate}.
\end{proof}

The generalization bounds in Theorem \ref{generalization_bounds_intermediate} are derived for the case when the parameter $\sigma$ goes to zero in accordance with the sample size. When the parameter $\sigma$ diverges, generalization bounds can be also derived as shown below.
\begin{theorem}\label{generalization_bounds_intermediateII}
Suppose that Assumption \ref{assum:bound_capacity} holds, $f^\star\in\mathcal{H}$, and the risk functional $\mathcal{R}^\sigma(\cdot)$ is defined in association with a calibrated modal regression kernel $K_\sigma$ and the corresponding representing function $\phi$. Let $f_{\mathbf{z},\sigma}$ be produced by \eqref{relaxation} with $\sigma>1$. Then for any $0<\delta<1$, with probability at least $1-\delta$, it holds that	
\begin{align*}
\mathcal{R}^\sigma(f^\star) - \mathcal{R}^\sigma(f_{\mathbf{z},\sigma}) \lesssim  \frac{\log \delta^{-1}}{\sigma \sqrt{n}}.  
\end{align*}
\end{theorem}	
\begin{proof} 
Similar to the proof of Theorem \ref{generalization_bounds_intermediate}, the desired bound can be established by applying Lemma   \ref{lem:concentration} to the random variable $\xi$ in \eqref{RV} with the only difference in bounding $\mathbb{E}\xi^2$. Recall that for a calibrated modal regression kernel $K_\sigma$, its representing function $\phi$ is bounded. Therefore, we have 
\begin{align*} 
\mathbb{E}\xi^2  =  \mathbb{E}\left[\frac{1}{\sigma}\phi \left(\frac{Y - f_{\mathcal{H},\sigma}(X)}{\sigma} \right) -  \frac{1}{\sigma}\phi \left(\frac{Y - f(X)}{\sigma} \right) \right]^2 \lesssim \sigma^{-2}.
\end{align*} 
In order to accomplish the proof, it suffices to apply Lemma \ref{lem:concentration} to the random variable $\xi$ with $B= 2\|\phi\|_\infty/\sigma$, $\gamma=0$, $c_p=\sigma^{-2p}$, and $c_\gamma=\sigma^{-2}$. By following the same procedure, the desired conclusion in Theorem \ref{generalization_bounds_intermediateII} can be obtained.
\end{proof}

The ERM learning scheme \eqref{relaxation} is adaptive in that the scale parameter $\sigma$ may vary in correspondence to the sample size $n$, e.g., $\sigma=n^{\theta}$ with $\theta\in\mathbb{R}$. Note from Theorem \ref{generalization_bounds_intermediateII} that, with a properly chosen $\sigma$ value, the ERM scheme \eqref{relaxation} is generalization consistent in the sense that the generalization risk $\mathcal{R}^\sigma(f_{\mathbf{z},\sigma})$ converges to $\mathcal{R}^\sigma(f^\star)$ when the sample size $n$ tends to infinity. It is also interesting to note that a wide range of $\sigma$ values is admitted to ensure such a consistency property as shown in the following corollary. 
\begin{corollary}\label{corollary_consistency}
Suppose that Assumption \ref{assum:bound_capacity} holds, $f^\star\in\mathcal{H}$, and the risk functional $\mathcal{R}^\sigma(\cdot)$ is defined in association with a calibrated modal regression kernel $K_\sigma$ and the representing function $\phi$. Let $f_{\mathbf{z},\sigma}$ be produced by \eqref{relaxation}. Then for any $0<\delta<1$, with probability at least $1-\delta$, it holds that	
\begin{align*}
\mathcal{R}^\sigma(f^\star) - \mathcal{R}^\sigma(f_{\mathbf{z},\sigma}) \rightarrow 0,  
\end{align*}
when $n\rightarrow +\infty$ and $\sigma:=n^\theta$ with $\theta\in \left(-\frac{2}{2+3p},+\infty\right)$. 
\end{corollary}

Corollary \ref{corollary_consistency} is an immediate result of Theorems \ref{generalization_bounds_intermediate} and \ref{generalization_bounds_intermediateII} and its proof is omitted here. With a properly chosen $\sigma$ value, the following conclusion reveals that the ERM scheme \eqref{relaxation} is also modal regression consistent. This gives an affirmative answer to Question $2$ listed in Subsection \ref{Subsection::three_blocks}.

\begin{theorem}\label{modal_regression_consistent}
Suppose that Assumptions \ref{smooth_density}, \ref{assum:bound_capacity} hold, and $f^\star\in\mathcal{H}$. Let $f_{\mathbf{z},\sigma}$ be produced by \eqref{relaxation} which is induced by a calibrated modal regression kernel $K_\sigma$ with $\sigma=\mathcal{O}(n^{-\frac{2}{10+3p}})$. For any $0<\delta<1$, with probability at least $1-\delta$, it holds that
\begin{align*}
\mathcal{R}(f_{\mathbf{z},\sigma})-\mathcal{R}(f^\star)\lesssim n^{-\frac{4}{10+3p}}\log (\delta^{-1}).
\end{align*} 
\end{theorem}
\begin{proof}
Since Assumption \ref{assum:bound_capacity} holds, $f^\star\in\mathcal{H}$, and $K_\sigma$ is a calibrated modal regression kernel, from Theorem \ref{generalization_bounds_intermediate} we know that for any $0<\delta<1$, with probability at least $1-\delta$, we have
\begin{align*}
\mathcal{R}^\sigma(f^\star) - \mathcal{R}^\sigma(f_{\mathbf{z},\sigma}) \lesssim  \left(\frac{1}{n\sigma}+\frac{\sigma^{-\frac{2+3p}{4}}}{n^{1/2}}+\frac{\sigma^{-\frac{2+3p}{2+p}}}{n^{\frac{2}{2+p}}}\right)\log\left(\frac{1}{\delta}\right).  
\end{align*}
When Assumption \ref{smooth_density} holds and $K_\sigma$ is a calibrated modal regression kernel, Theorem \ref{modal_regression_calibration} yields
\begin{align*}
\Big|\{\mathcal{R}(f^\star)-\mathcal{R}(f_{\mathbf{z},\sigma})\}-\{\mathcal{R}^\sigma(f^\star)-\mathcal{R}^\sigma(f_{\mathbf{z},\sigma})\}\Big| \lesssim \sigma^2. 
\end{align*}
As a result, for any $0<\delta<1$, with probability at least $1-\delta$, we have
\begin{align*}
\mathcal{R}(f^\star)-\mathcal{R}(f_{\mathbf{z},\sigma})  \lesssim \sigma^2 + \left(\frac{1}{n\sigma}+\frac{\sigma^{-\frac{2+3p}{4}}}{n^{1/2}}+\frac{\sigma^{-\frac{2+3p}{2+p}}}{n^{\frac{2}{2+p}}}\right)\log\left(\frac{1}{\delta}\right). 
\end{align*}
With the choice $\sigma=\mathcal{O}(n^{-\frac{2}{10+3p}})$, the proof of Theorem \ref{modal_regression_consistent} can be accomplished.
\end{proof}
 
\subsection{Towards the Function Estimation Calibration Problem}\label{subsec::function_calibration}
We now explore the relation between the modal regression consistency of $f_{\mathbf{z},\sigma}$ and its estimation consistency, which is termed as function estimation calibration problem in our study. From the studies in \cite{heinrich2013mode,dearborn2018indirect}, we realized that without further distributional assumptions, it is in general hopeless to learn the conditional mode through ERM approaches. In our study, we need to impose some further assumptions on the conditional density $p_{\epsilon|X}$ (see e.g., \cite{doss2016global}). 

\begin{definition}[Strongly $s$-Concave Density]\label{strongly_concave_density}
A density $p$ is \textbf{strongly $s$-concave} if it exhibits one of the following forms:
\begin{enumerate}
\item $p=\varphi_+^{1/s}$ for some strongly concave function $\varphi$ if $s>0$, where $\varphi_+=\max\{\varphi,0\}$;
\item $p=\exp(\varphi)$ for some strongly concave function $\varphi$ if $s=0$;
\item $p=\varphi_+^{1/s}$ for some strongly convex function $\varphi$ if $s<0$.
\end{enumerate}
\end{definition}

\begin{assumption}\label{assum:density}
The density of $\epsilon$ conditioned on $\mathcal{X}$, denoted by $p_{\epsilon|X}(\cdot| X)$, satisfies the following conditions:
\begin{enumerate}
\item $\sup_{x\in\mathcal{X}} p_{\epsilon|X}(0| X=x)=c_3$;
\item $p_{\epsilon|X}(t\mid X=x) \leq p_{\epsilon|X}(0| X=x)$, $\forall t\in\mathbb{R}, x\in\mathcal{X}$;
\item $\inf_{t\in[-2M, 2M],\,x\in\mathcal{X}}p_{\epsilon|X}(t| X=x) = c_0 >0$;
\item $p_{\epsilon|X}(\cdot\mid X)$ denotes strongly $s$-concave densities for all realizations of $X$.
\end{enumerate}
\end{assumption}

Conditions $1$ and $2$ in Assumption \ref{assum:density} require that the global mode of the conditional density $p_{\epsilon|X}$ for any realization of $X$ in $\mathcal{X}$ is uniquely zero while Condition $3$ rules out densities that are not bounded away from below in the vicinity of this unique mode. The first two conditions hold for continuous densities with a unique global mode. Condition $4$ assumes the strongly $s$-concave density assumption on $p_{\epsilon|X}$, which is typical from a statistical viewpoint as it holds for common symmetric and skewed distributions. Several representative examples are listed below:
\begin{example}[Student's $t$-distribution]
Let $\rho$ be a Student's $t$-distribution. Its probability density function $p$ is 
\begin{align*}
p(t)=\frac{\Gamma(\frac{\nu+1}{2})}{\Gamma(\frac{\nu}{2})}\left(1+\frac{t^2}{\nu}\right)^{-\frac{\nu+1}{2}},
\end{align*}
where $\nu$ is the number of degrees of freedom and $\Gamma$ is the gamma function. Specifically, when $\nu=1$, it gives the density function of a typical heavy-tailed distribution, namely, Cauchy distribution; when $\nu=\infty$, it is the density function of a most common probability distribution, i.e., Gauss distribution. One can easily see that for Student's $t$-distributions, their densities are strongly $s$-concave and are of the form $3$ in Definition \ref{strongly_concave_density}.
\end{example}

\begin{example}[Skewed normal distribution]\label{example::skewed_noise}
Let $\rho$ be a skewed normal distribution with the probability density function
\begin{align*}
p(t|\mu,\theta,\tau)=\frac{4\tau(1-\tau)}{\sqrt{2\pi\theta^2}}\exp\left\{-\frac{2(x-\mu)^2}{\sigma^2}\left(\tau-{\sf I}_{(x\leq \mu)}(x)\right)\right\},
\end{align*} 
where ${\sf I}_A(x)$ is the indicator function that takes the value $1$ if $A$ is true and $0$, otherwise. Clearly, the above density is also strongly $s$-concave and is of the form $2$ in Definition \ref{strongly_concave_density}.  
\end{example} 

When Assumption  \ref{assum:density} holds, the function estimation convergence can be elicited from the convergence of the modal regression risk, as shown in the following theorem.
\begin{theorem}\label{estimation_calibration}
Suppose that Assumption \ref{assum:density} holds and let $f:\mathcal{X}\rightarrow\mathcal{Y}$ be a measurable function in $\mathcal{H}$. Then, it holds that
\begin{align*}
\|f -f^\star\|_{L_{\rho_\mathcal{X}}^2}^2\lesssim \mathcal{R}(f^\star)-\mathcal{R}(f).
\end{align*}
\end{theorem}
\begin{proof} 
If Assumption \ref{assum:density} is fulfilled, then $p_{\epsilon|X}$ is strongly $s$-concave. We verify the desired relation by discussing different cases of $s$. If $s=0$, we know that $-\log p_{\epsilon|X}$ is strongly convex for all $x$. Consequently, in this case, it holds that
\begin{align*}
\begin{split}
\hspace{-1.7cm}\|f - f^\star\|^2_{L^2_{\rho_X}} &\lesssim \int_\mathcal{X}[-\log p_{\epsilon|X}(f(x) - f^\star(x)\mid X=x) + \log p_{\epsilon|X}(0\mid X=x)] \mathrm{d}\rho_{\mathsmaller{\mathcal{X}}}(x)\\
&\lesssim \int_\mathcal{X}  \left[ p_{\epsilon|X}(0\mid X=x) - p_{\epsilon|X}(f(x)-f^\star(x)\mid X=x)\right]\mathrm{d}\rho_{\mathsmaller{\mathcal{X}}}(x),
\end{split}
\end{align*}
where the last inequality is a consequence of the mean value theorem and Assumption \ref{assum:density}. If $s>0$, $-p_{\epsilon|X}^s$ is strongly convex  for all $x$, then
\begin{align*}
\begin{split}
\|f - f^\star\|^2_{L^2_{\rho_X}} &\lesssim \int_\mathcal{X}[- p_{\epsilon|X}^s(f(x) - f^\star(x)\mid X=x) + p_{\epsilon|X}^s(0\mid X=x)]\mathrm{d}\rho_{\mathsmaller{\mathcal{X}}}(x)\\
&\lesssim \max\{sc_0^{s-1}, sc_3^{s-1}\} \int_\mathcal{X} \left[ p_{\epsilon|X}(0\mid X=x) - p_{\epsilon|X}(f(x)-f^\star(x)\mid X=x)\right]\mathrm{d}\rho_{\mathsmaller{\mathcal{X}}}(x),
\end{split}
\end{align*}
where the second inequality is due to the Lipschitz continuity of $h(t)=t^s$ and Assumption \ref{assum:density}. If $s<0$, $p_{\epsilon|X}^s$ is strongly convex  for all $x$. In this case, we have 
\begin{align*}
\begin{split}
\hspace{-1.7cm}\|f - f^\star\|^2_{L^2_{\rho_X}} &\lesssim \int_\mathcal{X}[p_{\epsilon|X}^s(f(x) - f^\star(x)\mid X=x) - p_{\epsilon|X}^s(0\mid X=x)]\mathrm{d}\rho_{\mathsmaller{\mathcal{X}}}(x)\\
&\lesssim -sc_0^{s-1} \int_\mathcal{X} \left[ p_{\epsilon|X}(0\mid X=x) - p_{\epsilon|X}(f(x)-f^\star(x)\mid X=x)\right]\mathrm{d}\rho_{\mathsmaller{\mathcal{X}}}(x), 
\end{split} 
\end{align*}
where the second inequality is again due to the Lipschitz continuity of $h(t)=t^s$ and Assumption \ref{assum:density}. Recalling the fact that 
\begin{align*}
\begin{split}
\mathcal{R}(f^\star) - \mathcal{R}(f) = \int_\mathcal{X} \left[ p_{\epsilon|X}(0\mid X=x) - p_{\epsilon|X}(f(x)-f^\star(x)\mid X=x)\right]\mathrm{d}\rho_{\mathsmaller{\mathcal{X}}}(x),
\end{split}
\end{align*}
we complete the proof of Theorem \ref{estimation_calibration}.
\end{proof}

Combining the estimates established in the above several subsections, we are now able to answer Question $3$ raised in Subsection \ref{Subsection::three_blocks}.
\begin{theorem}\label{convergence_rates}
Suppose that Assumptions \ref{smooth_density}, \ref{assum:bound_capacity}, and \ref{assum:density} hold, and $f^\star\in\mathcal{H}$. Let $f_{\mathbf{z},\sigma}$ be produced by \eqref{relaxation} which is induced by a calibrated modal regression kernel $K_\sigma$ with $\sigma =\mathcal{O}(n^{-\frac{2}{10+3p}})$.  For any $0<\delta<1$, with probability at least $1-\delta$, we have
\begin{align*}
\|f_{\mathbf{z},\sigma} -f^\star\|_{L_{\rho_\mathcal{X}}^2}^2 \lesssim  n^{-\frac{4}{10+3p}}\log (\delta^{-1}).
\end{align*}
\end{theorem}
\begin{proof}
The theorem can be proved by combining the estimates in Theorems \ref{modal_regression_consistent} and \ref{estimation_calibration}.
\end{proof}

\subsection{Some Remarks}\label{subsec::comparison}
We give some remarks here. As noted earlier, most of the existing studies on modal regression were conducted by resorting to maximizing the joint density  estimator or the conditional density estimator. However, there are two main barriers when seeking the maximizer in this way. First, from a statistical learning viewpoint, learning the maximizer of the joint density or the conditional density is a local type learning scheme, in which one has to train the model for each test point. Second, the estimation of a high-dimensional joint or conditional density may suffer from the curse of dimensionality. In our proposed ERM approach to modal regression, the hypothesis space $\mathcal{H}$ is a function space that can be infinite-dimensional. In practice, it can be specified by applying certain regularization procedures. Moreover, the prevalent kernel-based methods can be naturally integrated since the hypothesis space can be chosen as a subset of a certain reproducing kernel Hilbert space. On the other hand, the proposed ERM approach to modal regression only involves a one-dimensional density estimation problem. From the above comparisons and the learning theory analysis conducted in this paper, it is easy to see that our study provides a different take on modal regression and the proposed ERM approach distinguishes our work with the existing studies.

\section{Modal Regression Interpretation of Correntropy based Regression}\label{sec::correntropy}
As mentioned above, our study on modal regression in this paper is initiated to understand the so-called maximum correntropy criterion in regression problems \citep[see][]{liu2007correntropy,principe2010information}. In this sense, the present study is a continuation of our previous work in \cite{fenglearning}. As a generalized correlation measurement, correntropy has been drawing much attention recently. Owing to its prominent merits on robustness, it has been pervasively used and has found many real-world applications in signal processing, machine learning, and computer vision \citep[see e.g.,][]{bessa2009entropy,he2011robust,he20122,lu2013correntropy,chen2016generalized}.

\subsection{Correntropy and Correntropy based Regression}
Mathematically speaking, correntropy is a generalized similarity measure between two scalar random variables $U$ and $V$, which is defined by $\mathcal {R}^\sigma(U,V)=\mathbb{E} K_\sigma(U,V)$. Here $K_\sigma$ is a Gaussian kernel given by $K_\sigma(u,v)=\exp\left\{-(u-v)^2/\sigma^2\right\}$ with the bandwidth $\sigma>0$, $(u,v)$ being a realization of $(U,V)$.
Given a set of i.i.d observations $\mathbf{z}=\{(x_i,y_i)\}_{i=1}^n$, for any $f:\mathcal{X}\rightarrow \mathbb{R}$, the empirical estimator of the correntropy between $f(X)$ and $Y$ is given as
\begin{align*} 
\mathcal{R}_n^{\sigma}(f):=\frac{1}{n}\sum_{i=1}^n K_\sigma(y_i,f(x_i)).
\end{align*} 

The maximum correntropy criterion based regression models the empirical target function by maximizing the empirical estimator of the correntropy $\mathcal{R}^\sigma$ as follows 
\begin{align}\label{MCCR_original} 
f_{\mathbf{z},\sigma}=\arg\max_{f\in\mathcal{H}}\mathcal{R}_n^{\sigma}(f),
\end{align} 
where $\mathcal{H}$ is assumed to be a compact subset of $C(\mathcal{X})$. Here, $C(\mathcal{X})$ is denoted as the Banach space of continuous functions on $\mathcal{X}$. The maximum correntropy criterion in regression problems has shown its efficiency for cases where non-Gaussian noise or outliers are present \citep[see e.g.,][]{liu2007correntropy,principe2010information,wang2013robust}.

In the literature, existing understanding of the maximum correntropy criterion and MCCR is still limited. More frequently, the maximum correntropy criterion is roughly taken as a robustified least squares criterion, analogously to the trimmed least squares criterion. However, the statistical performance of $f_{\mathbf{z},\sigma}$ and its relation to the least squares criterion are not clear. The barriers are mainly caused by the presence of the scale parameter $\sigma$ and the non-convexity of the related model. Recently, some theoretical understanding towards the maximum correntropy criterion was conducted in \cite{fenglearning} by introducing a distance-based regression loss, the study of which is inspired by those on information theoretic learning in \cite{hu2013learning} and \cite{fan2014consistency}. The main conclusion drawn in \cite{fenglearning} is that MCCR is essentially robustified mean regression with diverging $\sigma$ values. On the other hand, our study conducted in this paper shows that with diminishing $\sigma(n)$ values, MCCR is, in fact, modal regression. The built-in robustness of modal regression schemes may explain the empirical successes of MCCR from a different viewpoint.

\subsection{A General Picture of Correntropy based Regression}\label{sec::back_to_MCCR}
Based on this study and the study in \cite{fenglearning}, we are now able to depict a general picture of the correntropy based regression from a statistical learning viewpoint. To this end, we exposit the correntropy based regression by considering three different cases below, namely, (1): $\sigma=\sigma(n)\rightarrow \infty$; (2): $\sigma:=\sigma_0$ for some $\sigma_0>0$, that is, $\sigma$ is fixed and independent of the sample size $n$; (3): $\sigma:=\sigma(n)\rightarrow 0$. Before proceeding, we recall the following data-generating model 
\begin{align*}
Y=f^\star(X)+\epsilon.
\end{align*}

We first consider the case when $\sigma(n)\rightarrow \infty$. Under the zero-mean noise assumption on $\epsilon$, i.e., $\mathbb{E}(\epsilon|X)=0$, MCCR \eqref{MCCR_original} with $\sigma(n)\rightarrow \infty$ encourages the approximation of $f_{\mathbf{z},\sigma}$ towards the conditional mean function $\mathbb{E}(Y|X)$ and the scale parameter $\sigma$ in this case plays a trade-off role between robustness and generalization. More explicitly, in this case, MCCR is mean regression calibrated in the sense of the following theorem, see also Lemma 7 in \cite{fenglearning}:
\begin{theorem}[Lemma 7, \cite{fenglearning}]
Assume that $\mathbb{E} Y^4<\infty$ and denote $f^\star=\mathbb{E}(Y|X)$. For any $f\in\mathcal{H}$, it holds that 
\begin{align*}
\left|\|f -f^\star\|_{L_{\rho_\mathcal{X}}^2}^2-|\sigma^3(\mathcal{R}^\sigma(f^\star)-\mathcal{R}^\sigma(f))|\right|\lesssim \sigma^{-2}.
\end{align*}
\end{theorem}

\begin{figure}
\tikzset{trim left=0.0 cm}
\begin{minipage}{\textwidth}
\begin{tikzpicture}[xshift={0.5\textwidth-1.4cm}]
   \node [draw, circle, inner sep=1.4cm, label={30:$\mathcal{H}$}] (A) {};
   \node at (-1.35, -0.75) [draw, circle, inner sep=0.4mm, fill=black, label={-90:$f_{\mathbf{z},\sigma}$}] (D) {};
   \node at (0.8, 0.7) [draw, circle, inner sep=0.4mm, fill=black, label={90:$f_\mathcal{H}$}] (F) {};
     \node at (-0.2, 0.5) [draw, circle, inner sep=0.4mm, fill=black, label={90:$f_{\mathcal{H},\sigma}$}] (G) {};
   \node at (5.5, -1.2) [draw, circle, inner sep=0.4mm, fill=black, label={$\quad\quad f^\star$}] (E) {};
   \draw (F) -- (E)node[above,pos=0.45]{};
   \draw  (D) -- (G)node[above,pos=0.45]{};
   \draw  (G) -- (F)node[below,pos=0.50]{};
   \draw [dashed](D) -- (E);
\end{tikzpicture}
 \caption{A schematic illustration of the mechanism of correntropy-based regression when $\sigma(n)\rightarrow \infty$ and the noise variable $\epsilon$ is assumed to be zero-mean. $f_{\mathcal{H},\sigma}$ is the data-free counterpart of $f_{\mathbf{z},\sigma}$, $f_\mathcal{H}$ is the data-free least squares regression estimator and $f^\star$ is the conditional mean function $\mathbb{E}(Y|X)$.}
 \label{error_decomposition_ellsigma}
 \end{minipage}\\ ~~~~~
 
 \bigskip\medskip
 \bigskip\medskip
  \bigskip\medskip

 \begin{minipage}{\textwidth}
   \begin{tikzpicture}[xshift={0.5\textwidth-1.4cm}]
      \node [draw, circle, inner sep=1.4cm, label={30:$\mathcal{H}$}] (A) {};
      \node at (-1.35, -0.75) [draw, circle, inner sep=0.4mm, fill=black, label={-90:$f_{\mathbf{z},\sigma}$}] (D) {};
      \node at (0.8, 0.7) [draw, circle, inner sep=0.4mm, fill=black, label={90:$f_{\mathcal{H},\sigma}$}] (F) {};
      \node at (5.5, -1.2) [draw, circle, inner sep=0.4mm, fill=black, label={$\quad\quad f^\star$}] (E) {};
      \draw (F) -- (E)node[above,pos=0.45]{};
      \draw  (D) -- (F)node[above,pos=0.45]{	};
      \draw [dashed](D) -- (E);
   \end{tikzpicture}
    \caption{A schematic illustration of the mechanism of correntropy-based regression when $\sigma$ is fixed and independent on $n$ and the noise variable $\epsilon$ is assumed to be zero-mean. $f_{\mathcal{H},\sigma}$ is the data-free counterpart of $f_{\mathbf{z},\sigma}$ and $f^\star$ is the conditional mean function $\mathbb{E}(Y|X)$ or the conditional median function ${\sf{median}}(Y|X)$.}
    \label{depict_2}
 \end{minipage}\\ ~~~~~
 
 \bigskip\medskip
 \bigskip\medskip
 \bigskip\medskip

 \begin{minipage}{\textwidth}
  \begin{tikzpicture}[xshift={0.5\textwidth-1.4cm}]
     \node [draw, circle, inner sep=1.4cm, label={30:$\mathcal{H}$}] (A) {};
     \node at (-1.35, -0.75) [draw, circle, inner sep=0.4mm, fill=black, label={-90:$f_{\mathbf{z},\sigma}$}] (D) {};
     \node at (0.8, 0.7) [draw, circle, inner sep=0.4mm, fill=black, label={90:$f_{\mathcal{H},\sigma}$}] (F) {};
     \node at (5.5, -1.2) [draw, circle, inner sep=0.4mm, fill=black, label={$\quad\quad f^\star$}] (E) {};
     \draw (F) -- (E)node[above,pos=0.45]{};
     \draw  (D) -- (F)node[above,pos=0.45]{	};
     \draw [dashed](D) -- (E);
  \end{tikzpicture}
   \caption{A schematic illustration of the mechanism of correntropy-based regression when $\sigma(n)\rightarrow 0$ and the noise variable $\epsilon$ is assumed to admit a unique global zero-mode. $f_{\mathcal{H},\sigma}$ is the data-free counterpart of $f_{\mathbf{z},\sigma}$ and $f^\star$ is the conditional mode function ${\sf mode}(Y|X)$.}
   \label{depict}
 \end{minipage}
 \end{figure} 

It turns out that when $\sigma(n)$ is properly chosen with $\sigma(n)\rightarrow \infty$, the consistency of $\mathcal{R}^\sigma(f^\star) - \mathcal{R}^\sigma(f_{\mathbf{z},\sigma})$ implies the consistency of $\|f_{\mathbf{z},\sigma} -f^\star\|_{L_{\rho_\mathcal{X}}^2}^2$. Moreover, the following convergence rates are established in \cite{fenglearning}: 
\begin{theorem}
Assume that $f^\star=\mathbb{E}(Y|X)\in\mathcal{H}$ and $\mathbb{E} Y^4<+\infty$. Under a mild capacity assumption on $\mathcal{H}$, for any $0<\delta<1$, with confidence $1-\delta$, it holds that
\begin{align*}
\|f_{\mathbf{z},\sigma} -f^\star\|_{L_{\rho_\mathcal{X}}^2}^2\lesssim \sigma^{-2}+\sigma n^{-1/(1+p)},
\end{align*}
where the index $p>0$ reflects the capacity of the hypothesis space $\mathcal{H}$.
\end{theorem}

Obviously, according to the above theorem, when $\sigma$ is chosen as $\sigma:= n^{-1/(3+3p)}$,  the convergence rates for $\|f_{\mathbf{z},\sigma} -f^\star\|_{L_{\rho_\mathcal{X}}^2}^2$ of the type $\mathcal{O}\big(n^{-2/(3+3p)}\big)$ can be established.  It is worth to mention that recently, in \cite{fengWu2019learning}, the above moment condition $\mathbb{E}Y^4<+\infty$ was further relaxed to $\mathbb{E}|Y|^{1+\zeta}<+\infty$ with $\zeta>0$. Notice that in this case the underlying truth $f^\star$ corresponds to the conditional mean. Therefore, MCCR in this case is essentially robustified mean regression. A schematic illustration of MCCR in this case is given in Fig.\,\ref{error_decomposition_ellsigma}, in which $f_{\mathcal{H},\sigma}$ is the population version of $f_{\mathbf{z},\sigma}$ and $f_\mathcal{H}$ is the data-free least squares regression estimator. As argued in \cite{fenglearning}, compared with the least squares regression, an additional bias, i.e., the distance between $f_{\mathcal{H},\sigma}$ and $f_\mathcal{H}$, appears when bounding the $L_{\rho_\mathcal{X}}^2$-distance between $f_{\mathbf{z},\sigma}$ and the conditional mean function $\mathbb{E}(Y|X)$. Moreover, this bias in some sense reflects the trade-off between the convergence rate of $\|f_{\mathbf{z},\sigma} -f^\star\|_{L_{\rho_\mathcal{X}}^2}^2$ and the robustness of $f_{\mathbf{z},\sigma}$. These observations were further justified in a regularized learning setup in \cite{lv2019optimal}. 

\begin{table}[!ht] 
	\setlength{\tabcolsep}{7pt}
	\setlength{\abovecaptionskip}{5pt}
	\setlength{\belowcaptionskip}{5pt}
	\centering
	\captionsetup{justification=centering}
	\vspace{.5em}
	\begin{tabular}{@{}llll@{}}
		\toprule
		& $\sigma(n)\rightarrow \infty$ & $\sigma$ fixed & $\sigma(n)\rightarrow 0$  \\ \midrule
		resulting    &    conditional    & conditional mean      & conditional     \\
		estimator &   mean estimator & or median estimator & mode estimator \\ \midrule
		target  function & $\mathbb{E}(Y|X)$  & $\mathbb{E}(Y|X)$ or ${\sf median}(Y|X)$    & ${\sf mode}(Y|X)$  \\ \midrule
		noise  & weak moment  & bounded symmetric  & allow skewness    \\
	condition	&  condition &  or symmetric stable &    or heavy-tailedness \\ \midrule
		 rates & $\mathcal{O}(n^{-2/(3+3p)})$ & $\mathcal{O}(n^{-2/(2+p)})$ & $\mathcal{O}(n^{-4/(10+3p})$ \\
		\bottomrule
	\end{tabular}
	\caption{An overview of the three scenarios in correntropy based regression}\label{correntropy_table}
\end{table}

The case when $\sigma=\sigma_0$, i.e., $\sigma$ is fixed and independent of $n$, was investigated in \cite{fenglearning}, \cite{fengWu2019learning}, and \cite{feng2019learning}. As argued in \cite{fengWu2019learning}, with a fixed parameter $\sigma$ and without imposing any noise assumptions, it is impossible to learn the truth function $f^\star$. It turns out that in this case, if some noise assumptions are introduced, correntropy based regression regresses towards the conditional mean or the conditional median. More specifically, according to Lemma 18 in \cite{fenglearning}, under bounded symmetric noise assumptions, it is also calibrated mean regression when $\sigma_0$ is properly chosen. Convergence rates of $\|f_{\mathbf{z},\sigma} -f^\star\|_{L_{\rho_\mathcal{X}}^2}^2$ can be also established under such noise assumptions, see Theorem 6 in \cite{fenglearning}. Inspired by the work in \cite{fan2014consistency}, it is demonstrated in \cite{feng2019learning} that under the symmetric stable noise assumption, correntropy based regression can learn the underlying truth function $f^\star$ well where the truth function in this scenario corresponds to the conditional mean or the conditional median function.

The fact that MCCR can be cast as a modal regression problem when $\sigma(n)\rightarrow 0$ switches our attention from robust mean regression in \cite{fenglearning} to modal regression in this study. To recap, the modal regression scheme \eqref{relaxation} with the Gaussian kernel as the modal regression kernel retrieves MCCR \eqref{MCCR_original}. From the arguments in the preceding sections, we know that under the assumption that the noise variable admits a unique global zero-mode, MCCR \eqref{MCCR_original} with $\sigma(n)\rightarrow 0$ is modal regression calibrated. That is, under proper assumptions as listed in Theorem \ref{convergence_rates}, one may expect the learning theory type convergence from the MCCR estimator to the modal regression function ${\sf mode}(Y|X)$. Results reported in the above sections reveal that the modal regression problem can be also studied from an empirical risk minimization viewpoint. A schematic illustration of the mechanism of correntropy-based regression when $\sigma(n)\rightarrow 0$ is presented in Fig.\,\ref{depict}. In this case, the robustness of MCCR stems from the built-in robustness of modal regression estimators.  
 
An overview of the above-discussed three scenarios in correntropy based regression is summarized in Table \ref{correntropy_table}. To sum up, in short, what makes MCCR so special is that it results an interesting walk between modal regression and robustified mean regression by adjusting the scale parameter $\sigma$ in correspondence to the sample size $n$.   
 
\section{Model Selection and Numerical Validations}\label{sec::practical_issues}
 
This section is concerned with the implementation issues of the proposed ERM approach to modal regression. The model selection problem will be tackled by tailoring the technique of cross validation. Numerical validations on the effectiveness of the proposed  modal regression estimators will also be provided.

\subsection{Experimental Setup}
In our empirical studies, the hypothesis space $\mathcal{H}$ is chosen as a bounded subset of a reproducing kernel Hilbert space $\mathcal{H}_\mathcal{K}$ that is induced by a Mercer kernel $\mathcal{K}$. Specifically, we employ the following Tikhonov regularization to determine the radius of the working hypothesis space automatically:
\begin{align}\label{experiment::model}
f_{\mathbf{z},\sigma}:=\arg\min_{f\in{\mathcal{H}_\mathcal{K}\bigoplus\mathbb{R}}}\frac{1}{n}\sum_{i=1}^n \ell_
\sigma(y_i-f(x_i))+\lambda\|f\|_\mathcal{K}^2,
\end{align}
where $\ell_\sigma$ is the loss function $\ell_
\sigma(t)=\sigma^2 (1-\exp(-t^2/\sigma^2) )$, and $\lambda>0$ is a regularization parameter. The representor theorem ensures that $f_{\mathbf{z},\sigma}$ can be modeled by
\begin{align*}
f_{\mathbf{z},\sigma}(x)= \sum_{i=1}^n \alpha_{\mathbf{z},i} \mathcal{K}(x,x_i) + b_\mathbf{z},\,x\in\mathbb{R},
\end{align*} 
where $\alpha_{\mathbf{z}}=(\alpha_{\mathbf{z},1},\cdots, \alpha_{\mathbf{z},n} )^\top\in\mathbb{R}^n$ and $b_\mathbf{z}\in\mathbb{R}$ are learned from \eqref{experiment::model}. For the Mercer kernel $\mathcal{K}$, we use the Gaussian kernel $\mathcal{K}(x,x^\prime)=\exp\big(- \|x-x^\prime\|^2/h^2\big)$ with the bandwidth parameter $h>0$.  
 
\subsection{Algorithms} 
The regularization problem \eqref{experiment::model} is essentially a regularized M-estimation problem. We, therefore, apply the iteratively re-weighted least squares algorithm to solve it. The pseudo-code of the iteratively re-weighted least squares algorithm is listed in Algorithm \ref{alg:IRM}. For each iteration in Algorithm \ref{alg:IRM}, the weight is updated as follows:
\begin{align}\label{weight}
\omega_i^{k+1}= \frac{|\nabla\ell_\sigma(y_i-\boldsymbol{\mathcal{K}}_i^\top\boldsymbol{\alpha}^k-b^k)|}{|y_i-\boldsymbol{\mathcal{K}}_i^\top\boldsymbol{\alpha}^k-b^k|},\,\, i=1,\ldots,n,
\end{align}
with the initial guess $\boldsymbol{\alpha}^{ 0 }$, $b^0$ being zero.

\begin{algorithm}\caption{Iteratively Re-weighted Least Squares Algorithm for Solving \eqref{experiment::model}\label{alg:IRM}} 
\begin{algorithmic}
\STATE{\textbf{Input:} data $\{ (x_i, y_i) \}_{i=1}^n$,  regularization parameter $\lambda > 0$, Gaussian kernel bandwidth $h>0$, scale parameter $\sigma>0$ and the initial guess  $\boldsymbol{\alpha}^{ 0 }  \in\mathbb R^n$, $b^0\in\mathbb{R}$}.
\STATE{\textbf{Output:} the learned coefficient $\boldsymbol{\alpha}^{k+1} = (\alpha^{k+1}_1, \ldots,\alpha^{k+1}_n)^\top$} and $b^{k+1}\in\mathbb{R}$.
\WHILE{the stopping criterion is not satisfied}

\STATE{$\bullet$  Compute $\boldsymbol{\alpha}^{k+1}$ and $b^{k+1}$ by solving the following weighted least squares problem:
\begin{align*}
(\boldsymbol{\alpha}^{k+1},b^{k+1})= \arg\min_{\boldsymbol{\alpha}\in\mathbb{R}^n,\,b\in\mathbb{R}}\sum_{i=1}^n \omega_i^{k+1}(y_i-\boldsymbol{\mathcal{K}}_i^\top\boldsymbol{\alpha}-b)^2+\lambda\boldsymbol{\alpha}^\top\boldsymbol{\mathcal{K}}{\boldsymbol{\alpha}},
\end{align*} 
where $\omega_i^{k+1}$ is specified in \eqref{weight}.
}
\STATE{$\bullet$  Set $k:=k+1.$ }
 \ENDWHILE
\end{algorithmic}
\end{algorithm}

\subsection{Model Selection via Concatenated Cross Validation}
We now discuss the model selection problem of the proposed modal regression estimator. Here, the problem of model selection refers to the selection of the three tuning parameters, i.e., the regularization parameter $\lambda$, the bandwidth parameter $h$ of the Gaussian kernel, and the scale parameter $\sigma$ in the loss function. 
 
In our study, we choose these parameters by tailoring the frequently used cross-validation technique and  propose Concatenated Cross Validation (CCV) for model selection. In order to carry out the cross-validation process, we need to choose an error criterion. As we are interested in learning the conditional mode function, the mean squared error criterion, the absolute deviation error criterion, as well as the criteria under robustness constraints, see e.g., \cite{cantoni2001resistant}, may not serve well for this purpose. Recall that the ERM approach for modal regression we proposed in this study can be also re-expressed as follows
\begin{align*}
f_{\mathbf{z},\sigma}=\arg\max_{f\in\mathcal{H}}\frac{1}{n\sigma}\sum_{i=1}^n \exp\left({-\frac{(y_i-f(x_i))^2}{\sigma^2}}\right),
\end{align*}
where the hypothesis space $\mathcal{H}$ is chosen as a subset of a reproducing kernel Hilbert space induced by the Gaussian kernel as mentioned above. The criterion that we use in CCV is essentially the loss function in the above ERM scheme. More explicitly, denoting $\{(x_i,y_i)\}_{i=1}^m$ as the validation set and $\{\hat{y}_{i,\sigma}\}_{i=1}^m$ the estimated values, CCV can be proceeded through the following steps:

\medskip 
\noindent \textbf{Step 1:} We implement a first five-fold cross validation under the following criterion 
\begin{align*}
\arg\max_{\sigma} \frac{1}{m\sigma_0}\sum_{i=1}^m \exp\left({-\frac{(y_i-\hat{y}_{i,\sigma})^2}{\sigma_0^2}}\right),
\end{align*} 
where the initial value $\sigma_0$ is set as $m^{-1/5}$, which is the optimal $\sigma$ value according to our theoretical analysis. We denote the best $\sigma$ value selected in this step as $\sigma_1$. 

\medskip 
\noindent\textbf{Step 2:} We then implement a second five-fold cross validation under the following updated criterion 
\begin{align*}
\arg\max_{\sigma} \frac{1}{m\sigma_1}\sum_{i=1}^m \exp\left({-\frac{(y_i-\hat{y}_{i,\sigma})^2}{\sigma_1^2}}\right).
\end{align*} 
We denote the best $\sigma$ value selected in this step as $\sigma_2$.

\medskip 
\noindent\textbf{Step 3:} We continue to implement a third five-fold cross validation under the following updated criterion 
\begin{align*}
\arg\max_{\sigma} \frac{1}{m\sigma_2}\sum_{i=1}^m \exp\left({-\frac{(y_i-\hat{y}_{i,\sigma})^2}{\sigma_2^2}}\right).
\end{align*} 
We denote the best $\sigma$ value selected in this step as $\sigma_3$. Note that in the above steps, the estimated values $\{\hat{y}_{i,\sigma}\}_{i=1}^m$ also depend on the tuning parameters $\lambda$ and $h$, which are also updated accordingly at each step. We suppress the two subscripts for simplification.  

\medskip 
\noindent\textbf{Step 4:} With the selected $\sigma$ value in Step 3, we then train the regularized ERM model by using the iterative reweighted least squares algorithm. We then take the resulting estimator as the modal regression estimator and proceed with the prediction process. 

\begin{figure}[H] 
	\tikzset{trim left=0.0 cm}
	\setlength\figureheight{6cm}
	\setlength\figurewidth{4cm}
	\begin{minipage}[b]{0.3\textwidth}
		\input{mode_estimator.tikz}
	\end{minipage} 
	\caption{The dotted red curve with square marks is the conditional mode function $f_{\mathsmaller{\sf MO}}$ for observations generated by \eqref{toy_first} while the dotted black curve with plus marks gives the conditional mean function $f_{\mathsmaller{\sf ME}}$. The dotted blue curve with $\otimes$ marks represents the learned estimator $f_{\mathbf{z},\sigma}$ from noisy observations.}\label{toy_illustration_sigma}
\end{figure} 

\subsection{Numerical Validation on a Toy Example}\label{Toy_example}
We validate the effectiveness of the proposed modal regression estimator on the following toy example. We generate artificial data through the following regression model
\begin{align}\label{toy_first}
y=f^\star(x)+\kappa(x)\epsilon, 
\end{align}
where $x\sim U(0,1)$, $f^\star(x)=2\sin(\pi x)$, and $\kappa(x)=1+2x$. The noise variable is distributed as $\epsilon \sim 0.5 N(-1,4^2)+0.5N(1,0.1^2)$. A similar example was employed in \cite{yao2014new}. With simple calculations, it is easy to see that the conditional mean function is 
$f_{\mathsmaller{\sf ME}}=2\sin(\pi x)$ and the conditional mode function is approximately 
$f_{\mathsmaller{\sf MO}}=2\sin(\pi x)+1+2x$. In our experiment, $600$ observations are drawn from the above data-generating model and the size of the test set is also set to $600$. The reconstructed curve is plotted at the test points in Fig.\,\ref{toy_illustration_sigma}, in which the conditional mean function $f_{\mathsmaller{\sf ME}}$ and the conditional mode function $f_{\mathsmaller{\sf MO}}$ are also plotted for comparisons. In our experiment, we choose the three tuning parameters, i.e., the bandwidth parameter $h$ of the Gaussian kernel, the regularization parameter $\lambda$, and the scale parameter $\sigma$ in the loss function, by using Concatenated Cross Valudation described above.

From Fig.\,\ref{toy_illustration_sigma}, it is easy to see that the proposed modal regression estimator $f_{\mathbf{z},\sigma}$ can learn the conditional mode function $f_{\mathsmaller{\sf MO}}$ well instead of learning the conditional mean function $f_{\mathsmaller{\sf ME}}$. It is interesting to point out that the obtained empirical target function $f_{\mathbf{z},\sigma}$ can also learn the conditional mean function with a large $\sigma$ value as explained in Section \ref{sec::correntropy}.

\subsection{Application to  Speed-Flow Data}

We now apply the proposed modal regression estimator to speed-flow data. Speed-flow data are intensively discussed in transportation science, which are usually visualized  in terms of speed-flow diagrams. In this subsection, we apply the proposed modal regression approach to the analysis of the speed-flow data collected in \cite{petty1996freeway}, the speed-flow diagrams of which are presented in Figs.\,\ref{real_illustration_sigma1} and \ref{real_illustration_sigma2}. In the speed-flow diagrams, the $x$-axis is traffic flow that is measured in vehicles per lane per hour while the $y$-axis is speed measured in miles per hour. The speed-flow data analyzed here contain two data sets collected in 1993 on two individual lanes (lane 2 and lane 3) of the 4-lane Californian freeway I-880. The data were collected by loop detectors, and the time units are 30 seconds per observation, see \cite{einbeck2006modelling} for more background details. This speed-flow data  contains 1318 observations and are publicly available in the R-package {\sf{hdrcde}}. From the speed-flow diagrams, it can be observed that the mean regression function may not be able to characterize the functional relation between speed and traffic flow. This is also observed in many related studies that analyze the speed-flow data, see e.g., \cite{einbeck2006modelling}. This is because the less dense cloud of data points at the bottom of the two figures, which corresponds to situations where speed is dismissed, may be interpreted as abnormal observations when pursuing such a functional relation.

 \begin{figure}  
 	\tikzset{trim left=2.0 cm}
 	\setlength\figureheight{4cm}
 	\setlength\figurewidth{2cm}
 	\begin{minipage}[b]{0.2\textwidth}
 		\input{speedflow.tikz}
 	\end{minipage} 
 	\caption{The blue curve represents the conditional mode function estimator $f_{\mathbf{z},\sigma}$ for 1318 observations of lane 2 while the black curve gives the conditional mean function estimator by kernel ridge regression.}\label{real_illustration_sigma1}
 \end{figure}

 \begin{figure} 
	\tikzset{trim left=2.0 cm}
	\setlength\figureheight{4cm}
	\setlength\figurewidth{2cm}
	\begin{minipage}[b]{0.2\textwidth}
		\input{speedflow2.tikz}
	\end{minipage} 
	\caption{The blue curve represents the conditional mode function estimator $f_{\mathbf{z},\sigma}$ for 1318 observations of lane 3 while the black curve gives the conditional mean function estimator by kernel ridge regression.}\label{real_illustration_sigma2}
\end{figure} 

In our experiments, we apply the proposed modal regression approach to pursuing the functional relation. By following the same setup as in our above experiments on artificial data, we plot the learned modal regression estimator as well as the mean regression estimator resulting from kernel ridge regression. From the reported experimental results in Fig.\,\ref{real_illustration_sigma1} and Fig.\,\ref{real_illustration_sigma2}, it can be seen that modal regression estimator is less sensitive to abnormal observations and serves better in trend estimation when analyzing speed-flow data. It would be interesting to explore more real-world applications of the modal regression estimator learned through the proposed ERM approach, which will be the future work of our study in this respect.

\section{Conclusions}\label{sec::conclusion}
As one of the important regression protocols, modal regression has not been much studied yet in the statistical learning literature. In this study, we investigated the modal regression problem from a statistical learning viewpoint. By assuming the existence and the uniqueness of the global mode of the conditional distribution in regression, we reformulated the modal regression problem into the classical empirical risk minimization framework. In particular, such a reformulation renders the associated modal regression approach dimension-independent. A learning theory framework for analyzing and assessing the proposed modal regression estimator was also developed. Based on the proposed statistical learning treatment on modal regression, we gained some insights into the regression problem. These insights include: first, modal regression problem can be tackled via empirical risk minimization and can be also interpreted from a kernel density estimation point of view; second, learning for modal regression is generalization consistent and modal regression calibrated in the sense defined in our study; third, function estimation consistency and convergence in the sense of the $L_{\rho_\mathcal{X}}^2$-distance can be derived in modal regression. These findings in return unveil the working mechanism of MCCR when its scale parameter tends to zero as in this case, it corresponds to a modal regression problem.

\acks{The authors would like to thank the Action Editor and the reviewers for their constructive suggestions and comments that improved the quality of this paper. The research leading to these results has received funding from the European Research Council under the European Union's Seventh Framework Programme (FP7/2007-2013) / ERC AdG A-DATADRIVE-B (290923) and ERC AdG E-DUALITY (787960) under the European Union's Horizon 2020 research and innovation programme. This paper reflects only the authors' views, the Union is not liable for any use that may be made of the contained information. Research Council KUL: GOA/10/09 MaNet, CoE PFV/10/002 (OPTEC), BIL12/11T; PhD/Postdoc grants. Flemish Government:  FWO: projects: G.0377.12 (Structured systems), G.088114N (Tensor based data similarity); PhD/Postdoc grants. IWT: projects: SBO POM (100031); PhD/Postdoc grants. iMinds Medical Information Technologies SBO 2014. Belgian Federal Science Policy Office: IUAP P7/19 (DYSCO, Dynamical systems, control and optimization, 2012-2017). Yunlong Feng also gratefully acknowledges the support of Simons Foundation Collaboration Grant \#572064 and the Ralph E. Powe Junior Faculty Enhancement Award by Oak Ridge Associated Universities. The research of Jun Fan was supported in part by the Hong Kong RGC Early Career Schemes 22303518, and the NSF grant of China (No. 11801478). The corresponding author is Jun Fan.}

\bibliography{FENGBib}  
 
\end{document}